\documentclass{article}

\usepackage{arxiv}

\usepackage[utf8]{inputenc} 
\usepackage[T1]{fontenc}    
\usepackage{hyperref}       
\hypersetup{
    colorlinks,
    linkcolor={red!50!black},
    citecolor={blue!50!black},
    urlcolor={magenta!100!black}
}
\usepackage{url}            
\usepackage{booktabs}       
\usepackage{amsfonts}       
\usepackage{nicefrac}       
\usepackage{microtype}      
\usepackage{lipsum}
\usepackage{graphicx}
\graphicspath{ {./images/} }
\usepackage[sort,comma,round]{natbib}
\usepackage{amsmath} 
\usepackage[capitalize,noabbrev]{cleveref}  
\usepackage{amssymb}            
\usepackage{mathtools}          
\usepackage{amsthm}   
\usepackage{bm} 
         
\usepackage{xcolor}             
\usepackage{mdframed}           
\usepackage{multirow}           
\usepackage{multicol}           
\usepackage{listings}           
\usepackage{subcaption}         
\usepackage{adjustbox}          
\usepackage{xspace}             
\usepackage{diagbox}            

\usepackage{makecell}           
\usepackage{colortbl}           

\theoremstyle{plain}            
\newtheorem{theorem}{Theorem}[section] 

\definecolor{darkred}{rgb}{0.55, 0, 0}

\newcommand{\intrprmpt}[1]{\textcolor{blue}{\{#1\}}}

\lstdefinelanguage{prompt}{
    basicstyle=\ttfamily,
    breaklines=true,
    columns=fullflexible,
    mathescape=false,
    escapeinside={(@}{@)},
    breakindent=0pt,
    postbreak=\mbox{\tiny\textcolor{darkred}{$\hookrightarrow$}\space},
}

\colorlet{punct}{red!60!black}
\definecolor{background}{HTML}{EEEEEE}
\definecolor{delim}{RGB}{20,105,176}
\colorlet{numb}{magenta!60!black}

\lstdefinelanguage{json}{
    basicstyle=\ttfamily,
    stepnumber=1,
    numbersep=8pt,
    showstringspaces=false,
    breaklines=true,
    postbreak=\mbox{\tiny\textcolor{darkred}{$\hookrightarrow$}\space},
    literate=
     *{0}{{{\color{numb}0}}}{1}
      {1}{{{\color{numb}1}}}{1}
      {2}{{{\color{numb}2}}}{1}
      {3}{{{\color{numb}3}}}{1}
      {4}{{{\color{numb}4}}}{1}
      {5}{{{\color{numb}5}}}{1}
      {6}{{{\color{numb}6}}}{1}
      {7}{{{\color{numb}7}}}{1}
      {8}{{{\color{numb}8}}}{1}
      {9}{{{\color{numb}9}}}{1}
      {:}{{{\color{punct}{:}}}}{1}
      {,}{{{\color{punct}{,}}}}{1}
      {\{}{{{\color{delim}{\{}}}}{1}
      {\}}{{{\color{delim}{\}}}}}{1}
      {[}{{{\color{delim}{[}}}}{1}
      {]}{{{\color{delim}{]}}}}{1}
      {\{input\}}{{{\color{blue}{\{input\}}}}}{6}
}

\theoremstyle{plain}

\newtheorem{lemma}[theorem]{Lemma}

\theoremstyle{definition}
\newtheorem{definition}[theorem]{Definition}
\newtheorem{assumption}[theorem]{Assumption}
\theoremstyle{remark}

\newcommand{\rand}{\texttt{Rand}\xspace}
\newcommand{\topk}{\texttt{TopK}\xspace}
\newcommand{\diversity}{\texttt{Div}\xspace}
\newcommand{\topkdiv}{\texttt{TopK-Div}\xspace}
\newcommand{\kmeans}{\texttt{K-Means}\xspace} 

\definecolor{cb_red}{RGB}{213,94,0}
\definecolor{cb_blue}{RGB}{0,114,178}
\definecolor{cb_yellow}{RGB}{240,228,66}
\definecolor{cb_gray}{RGB}{204,204,204}
\definecolor{cb_orange}{RGB}{230,159,0}
\definecolor{cb_skyblue}{RGB}{86,180,233}
\definecolor{cb_green}{RGB}{0,158,115}
\definecolor{cb_purple}{RGB}{204,121,167}

\newcommand{\hlcella}{\cellcolor{cb_gray!40}}

\makeatletter
\newcommand*{\rom}[1]{\expandafter\@slowromancap\romannumeral #1@} 
\makeatother

\usepackage[textsize=tiny]{todonotes}

\usepackage{amsmath,amsfonts,bm}

\def\eqref#1{equation~\ref{#1}}

\def\1{\bm{1}}

\DeclareMathAlphabet{\mathsfit}{\encodingdefault}{\sfdefault}{m}{sl}
\SetMathAlphabet{\mathsfit}{bold}{\encodingdefault}{\sfdefault}{bx}{n}

\def\gA{{\mathcal{A}}}

\def\gD{{\mathcal{D}}}
\def\gE{{\mathcal{E}}}

\def\gQ{{\mathcal{Q}}}

\def\gT{{\mathcal{T}}}

\def\gX{{\mathcal{X}}}
\def\gY{{\mathcal{Y}}}

\newcommand{\E}{\mathbb{E}}

\title{The Role of Diversity in In-Context Learning for Large Language Models}

\author{Wenyang Xiao\thanks{First two authors contribute equally. WX conducts most of the experiments. HZ designs and prototypes most of the experiments.}\ \ \textsuperscript{1}\quad Haoyu Zhao\textsuperscript{*\ 2}\quad Lingxiao Huang\textsuperscript{1} \\
\textsuperscript{1} School of Computer Science and Technology, Nanjing Universtiy\\
\textsuperscript{2} Department of Computer Science \& Princeton Language and Intelligence (PLI), Princeton University\\
\texttt{wenyangxiao@smail.nju.edu.cn, haoyu@princeton.edu}
}

\begin{document}

\maketitle

\begin{abstract}
In-context learning (ICL) is a crucial capability of current large language models (LLMs), where the selection of examples plays a key role in performance. While most existing approaches focus on selecting the most \emph{similar} examples to the query, the impact of \emph{diversity} in example selection remains underexplored. We systematically investigate the role of \emph{diversity} in in-context example selection through experiments across a range of tasks, from sentiment classification to more challenging math and code problems. Experiments on Llama-3.1, Gemma-2, and Mistral-v0.3 families of models show that diversity-aware selection methods improve performance, particularly on complex tasks like math and code, and enhance robustness to out-of-distribution queries. To support these findings, we introduce a theoretical framework that explains the benefits of incorporating diversity in in-context example selection.
\end{abstract}

\section{Introduction}\label{sec:intro}

In-context learning (ICL)~\citep{brown2020language} has emerged as one of the most significant and versatile capabilities of large language models (LLMs). In this paradigm, a language model is given a prompt which consist of several in-context examples (demonstrations) and the query, and generates the output for the query without updating the parameters or computing the gradient. This feature enables LLMs to adapt to a wide range of tasks using limited resources.

Despite its importance and popularity, the effectiveness of retrieval-based in-context learning is highly sensitive to the selection of in-context examples 
~\citep{lu2021fantastically,liu2021makes,chang2023data}. 
To address this, prior work has explored various selection strategies: choosing examples most \emph{similar} to the query in embedding space~\citep{liu2021makes,DBLP:conf/aaai/YangGW0L0W22,wu2023self,qin2023context}, maximizing feature \emph{coverage}~\citep{levy2023diverse,ye2023complementary,gupta2023coverage}, selecting based on \emph{difficulty}~\citep{ma2025problemsolvinglogicguidedcurriculum,DBLP:conf/emnlp/SwayamdiptaSLWH20,yuan2025enhancingsampleselectionlabel,cook-etal-2025-simple}, or choosing examples based on \emph{sensitivity}~\citep{chen2023relation}. 
Other approaches train deep neural retrievers~\citep{karpukhin2020dense,rubin2022learning,luo2023dr,scarlatos2023reticl} or leverage feedback from large language models to guide selection~\citep{li2023finding,chen2023relation,wang2023large}.
Among these, similarity-based methods remain the fundamental baseline due to their conceptual simplicity and consistent empirical success. However, relying solely on similarity can lead to redundancy among demonstrations and potentially omit important but less similar features~\citep{levy2023diverse,gupta2023coverage}.


In contrast, the explicit role of \emph{diversity} in retrieval-based demonstration selection remains underexplored. 
Diversity has shown promise in other domains—such as fixed-prompt ICL with global demonstration sets~\citep{li2023findingsupportexamplesincontext,luo2024incontextlearningretrieveddemonstrations}, active learning~\citep{giouroukis2025dualdiversityuncertaintyactive,shi2016diversifying}, coreset construction~\citep{wan2024contributingdimensionstructuredeep,zhan2025coresetbasedtaskselectionsampleefficient,sener2018activelearningconvolutionalneural}, and instruction tuning~\citep{wang2024diversitymeasurementsubsetselection}. 
%
%
Although some recent work incorporates feature coverage as a proxy for diversity~\citep{levy2023diverse,ye2023complementary,gupta2023coverage}, coverage primarily aims at spanning input features rather than explicitly promoting representational variety. 
Conversely, explicit diversity-aware selection risks retrieving examples too dissimilar from the query, potentially harming performance~\citep{DBLP:conf/acl/AnLFC0L023}. 
Therefore, it remains unclear whether and when explicit diversity consideration is beneficial in retrieval-based ICL—especially for tasks lacking clear local structure. This naturally motivates the following fundamental questions:

\begin{quote}
    \emph{Should we explicitly consider diversity when selecting in-context examples? If so, under what conditions does it outperform similarity-based methods? And fundamentally, why does diversity help?}
\end{quote}

\subsection{Our contributions}

\begin{figure*}[t!]
    \centering
    \begin{subfigure}[t]{0.3\textwidth}
        \centering
        \includegraphics[width=\linewidth]{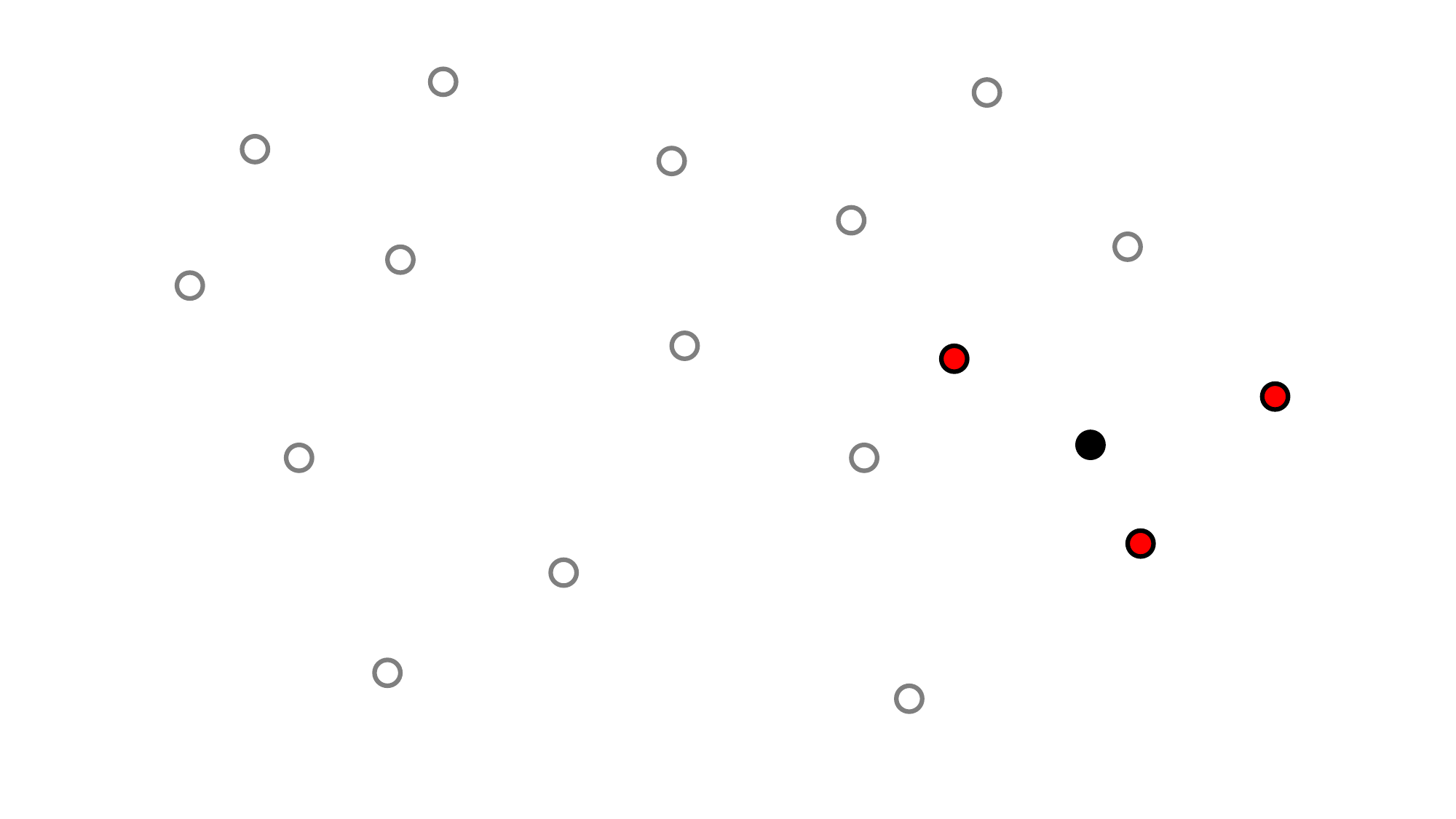}
        \caption{\topk selection}
    \end{subfigure}%
    \begin{subfigure}[t]{0.3\textwidth}
        \centering
        \includegraphics[width=\linewidth]{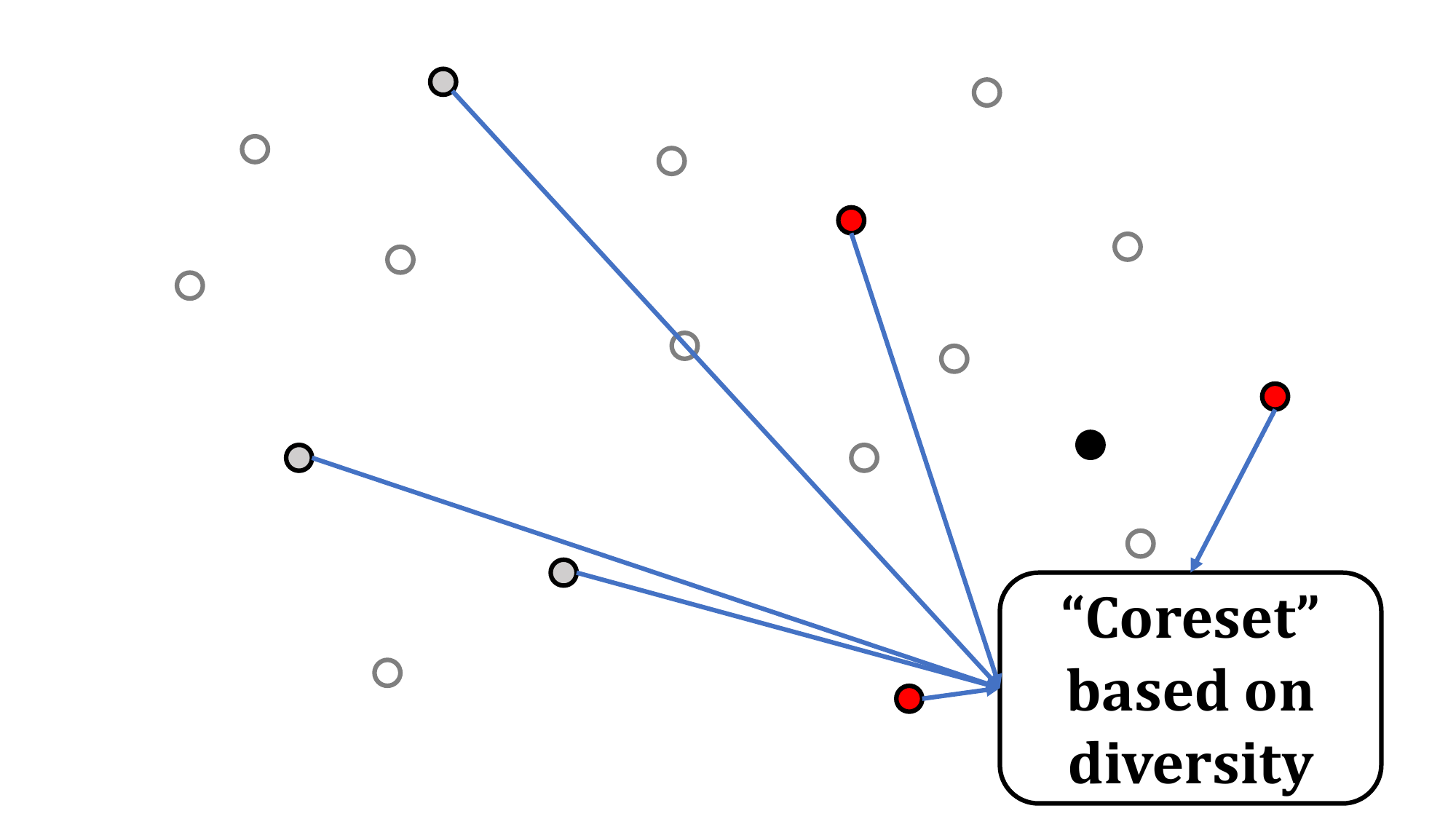}
        \caption{\diversity selection}
    \end{subfigure}
    \begin{subfigure}[t]{0.3\textwidth}
        \centering
        \includegraphics[width=\linewidth]{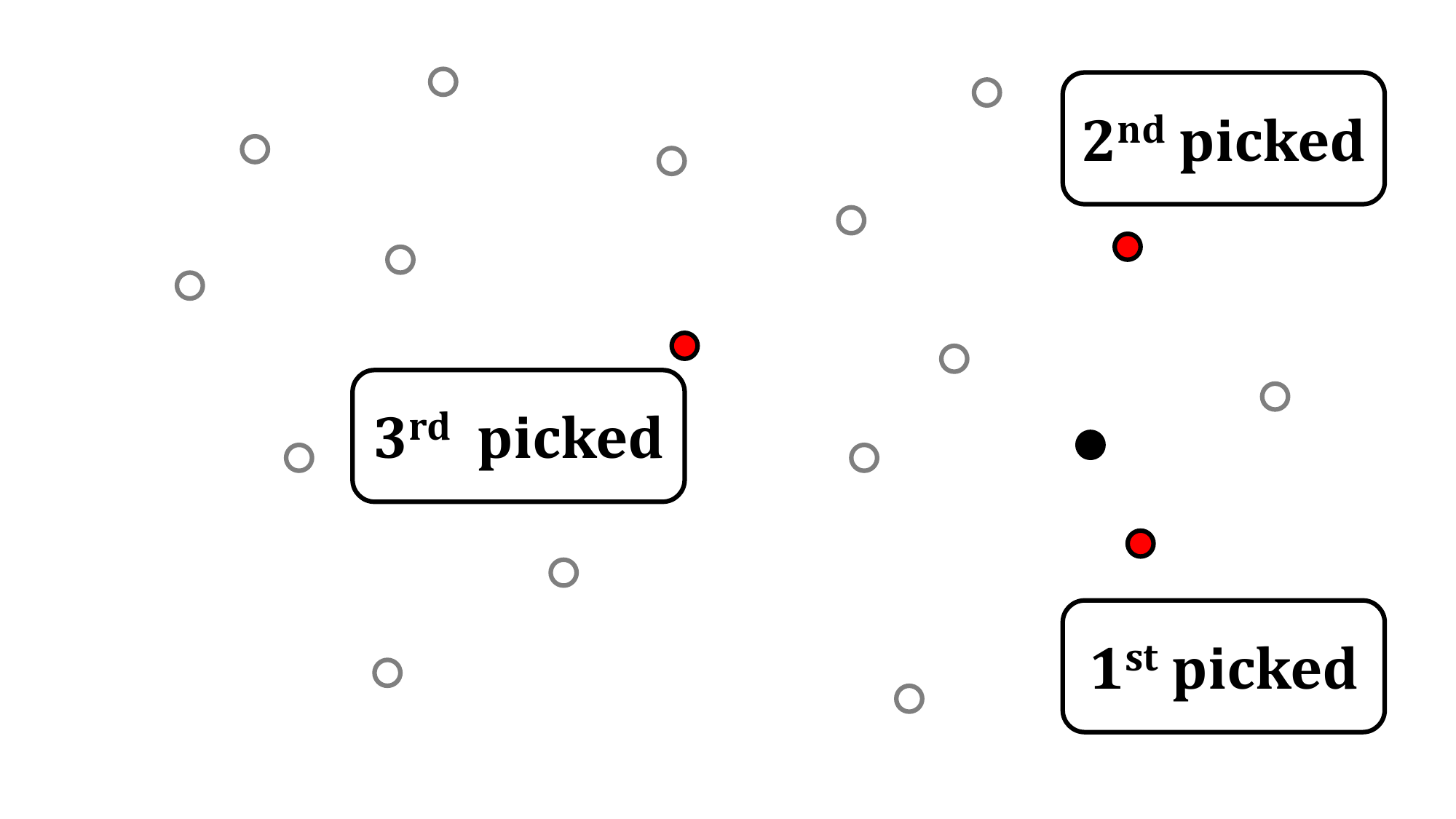}
        \caption{\topkdiv selection}
    \end{subfigure}
    \caption{An illustrative example for \topk, \diversity, and \topkdiv methods. Point filled in black denotes the query. \textbf{(a)} \topk: Select the most similar demonstrations (3 points filled in red) in the embedding space. \textbf{(b)} \diversity: First select a ``coreset'' based on some \emph{diversity} metric, which is fixed for all queries (6 points filled in gray or red). Then select the most similar demonstrations (3 points filled in red) among this ``coreset''. \textbf{(c)} \topkdiv: Select the demonstrations sequentially based on the linear combination of similarity to the query and the diversity with the selected examples. The first example is the one most similar to the query. When picking the second, it is balanced between the similarity to the query and the diversity from the first example.}
    \label{fig:illustration}
    \vspace{-0.2in}
\end{figure*}

We address the research question by systematically studying the role of diversity in in-context learning across tasks including sentiment classification, commonsense reasoning, math generation, reading comprehension, and SQL code generation. 
We compare four demonstration selection methods: (1) random selection (\rand); (2) selecting the $K$ most similar examples to the query (\topk)~\citep{liu2021makes}; (3) selecting similar examples from a diversity-reduced subset (\diversity)~\citep{su2023selective}, which relates to DPP-based diversity~\citep{chen2018fast}; and (4) a sequential method that balances similarity to the query and diversity among selected examples (\topkdiv). 
See \Cref{fig:illustration} for illustration and \Cref{sec:background} for formal definitions.

Through experiments on frontier open-source models (Llama-3.1~\citep{dubey2024llama}, Gemma-2~\citep{team2024gemma}, and Mistral-v0.3~\citep{jiang2023mistral}), we reach the following findings.

\underline{\textbf{Finding 1}:} \emph{Diversity-aware demonstration selection methods perform better on more ``challenging'' tasks like reading comprehension, math, and code.} While changing the tasks and even the language model to use will change the ranking of the demonstration selection methods we test, in general we find that diversity-aware methods, namely \diversity and \topkdiv perform better on more challenging tasks like reading comprehension, math, and simple code generation. 
On the other hand, on simple tasks like sentiment classification, \topk performs the best (\Cref{tab:main-res}).
Quantitative analysis of performance improvements under varying levels of added diversity demonstrates that more challenging tasks benefit more from increased diversity, further validating this finding (\Cref{tab:best knn diversity}).

\underline{\textbf{Finding 2}:} \emph{Diversity-aware methods works better for out-of-distribution queries.} When the query and demonstrations come from different distributions, diversity-aware methods are more likely to perform well. For example, on sentiment classification, when both demonstration and query come from the SST-2 dataset, which consists of movie reviews, \topk performs the best, and there is a gap with all other methods. However, when changing the demonstrations from SST-2 to Amazon, a shopping review dataset, \topkdiv outperforms \topk (\Cref{tab:ood-res}). 
A similar observation holds for various splits of Geoquery dataset (\Cref{fig:geoquery-splits}).

\underline{\textbf{Finding 3}:} In the same task, diversity-aware methods likely perform better on ``harder'' examples, e.g. reading comprehension with longer context, or SQL code generation with more structures (\Cref{tab:harder-example}). 

We discuss these findings in detail in \Cref{sec:main-exp}, where we also examine the distinction between diversity and coverage. 
Our analysis reveals that there are cases in which diversity exhibits a ``beyond-coverage'' phenomenon, both at the task level and the example level. 
We also conduct ablation studies by changing the model sizes (from 1B to 70B) and by varying the number of demonstrations. 
We discuss the robustness of our findings under different choices of hyperparameters.
Finally, we also present a theory to justify the use of diversity in in-context learning in \Cref{sec:theory}.
Overall, our findings, supported by theoretical justification, offer a deeper understanding of the role of diversity in demonstration selection for in-context learning.

\section{Background and notations}\label{sec:background}

In this section, we introduce the in-context learning (ICL) paradigm, relevant demonstration selection methods, and associated notations.

\textbf{In-context learning (ICL).}
A task $\gT=(\gX,\gY,P(y|x))$ defines a probabilistic mapping from an input $x \in \gX$ to an output $y \in \gY$. For example, the task can be sentiment classification where the input space contains reviews of products and the output space contains the customer’s corresponding sentiment (positive or negative). We are provided with a demonstration set $D = \{(x_i,y_i)\}_{i=1}^n$, where inputs $x_i$ are drawn from a demonstration distribution $\gD_\gX$ and $y_i \sim P_\gT (y|x_i)$. Queries $x_q$ are drawn from a query distribution $\gQ_\gX$, which may differ from $\gD_\gX$ (representing shifts in domain or complexity). As to math task, you might have many elementary school-level math problems at hand but you want to solve a math problem from middle school. Now given a query input $x_q\sim\gQ_\gX$, the in-context learning paradigm refers to the following capability of a large language model.

\begin{definition}[\textbf{In-Context learning (ICL)}]\label{defn:icl}
    Given an LLM, a prompting strategy \texttt{Prompt}, a demonstration set $D = \{(x_i,y_i)\}_{i=1}^n$, and a query $x_q$, ICL involves selecting a small subset $S=\{(x_{j_i},y_{j_i})\}_{i=1}^K $ with shots $K$ from the demonstrations $D$. The LLM then predicts the output $y_q$ for $x_q$ as:
    $P_\gT(y|x_q) \approx \text{LLM}(\texttt{Prompt}(S, x_q))$.
\end{definition}

\textbf{Demonstration selection for ICL.}
Choosing a small subset $S$ (see \Cref{defn:icl}) is vital due to LLM context limits, efficiency needs, and the observation that excessive demonstrations can impair performance. Prior work has shown that ICL performance is highly sensitive to this selection~\citep{liu2021makes}, and thus sparks the study for \emph{demonstration selection}. While numerous selection strategies are proposed , the most notable and effective methods are the
ones that select the demonstrations most similar to the query in the embedding space. Efforts are also made to retrieve the demonstrations using another model (can be another LLM), as well as considering diversity/coverage. However, there is no consensus on which method to use in a specific setting, and there is nearly no understanding of these methods (further discussed in \Cref{sec:related-works}). 

To analyze these methods and the role of diversity, we focus on four representative selection strategies:

\textbf{Method 1: \rand.} For a query $x_q$, this method uniformly and randomly selects $K$ demonstrations from the set $D$. 
Note that \rand can also be viewed as a method that is aware of
diversity, but it has nothing to do with the coverage.

\textbf{Method 2: \topk.} This method selects $K$ demonstrations from $D$ that exhibit the highest cosine similarity to the query $x_q$ within an embedding space mapped by $E:\gX\to\gE$, serving as a core similarity-based approach. 
It maximizes 
\begin{equation}   \label{eqa:simi}
\texttt{Similarity} \textstyle (E(x_i), E(x_q)) := \frac{\langle E(x_i), E(x_q)\rangle}{\|E(x_i)\|\cdot\|E(x_q)\|}.
\end{equation}

\textbf{Method 3: \diversity.} This approach first constructs a diverse ``coreset'' $D_r \subset D$ of size $m$ (where $K \le m \le n$). Starting with one randomly chosen demonstration, $D_r$ is built greedily by adding $(x,y) \in D \setminus D_r$ that maximizes 
\begin{equation} \label{eqa:div}    
\texttt{Diversity}\textstyle(E(x), D_r) := 1 - \frac{1}{|D_r|}\sum_{(x_j,y_j)\in D_r}\texttt{Similarity}(E(x),E(x_j)),
\end{equation}

we stop after $D_r$ contains $m$ examples. This is the procedure to get a diverse set of demonstrations for a task~\citep{su2023selective}. Subsequently, \topk selection is applied to $D_r$ to choose $K$ demonstrations for the query $x_q$. The coreset size $m$ interpolates between diversity (small $m$, e.g., $m=K$ will purely focus on diversity) and similarity (large $m$ makes it similar to \topk on $D$).

\textbf{Method 4: \topkdiv.} This method serves as a combination of \topk and \diversity, which includes some awareness of the diversity through similarity-based selection. It is also a greedy-like procedure when selecting the demonstration set $S$. Suppose that $S$ does not reach size $K$, then we select the demonstration $(x,y)\in D\setminus K$ that maximize the following metric:
\begin{equation} \label{eqa:topkdiv}
\alpha\cdot\texttt{Similarity}\textstyle(E(x),e_q) + (1-\alpha)\texttt{Diversity}(E(x),S), 
\end{equation}


Here, $\alpha$ is a hyperparameter that controls the trade-off between similarity and diversity. 
We stop when $S$ has size $K$. 
For the first demonstration (when $S$ is empty), $\texttt{Diversity}(E(x), S)$ is defined as $0$, thereby prioritizing similarity. 
\topkdiv can be viewed as incorporating a tunable level of diversity (controlled by $\alpha$) into the similarity-based approach \topk. 
Note that \diversity is a special case of \topkdiv with $\alpha = 0$.

See \Cref{fig:illustration} for an illustrative example of these selection methods. 

\section{Experiments and findings}\label{sec:main-exp}

This section empirically tests whether diversity-aware retrieval (\diversity, \topkdiv) yields more reliable in-context learning than similarity-only baselines (\topk).

\textbf{Tasks and datasets.} We consider 5 tasks: sentiment classification (classification task), commonsense reasoning (multi-choice), text to SQL generation (generation), math (generation), and reading comprehension (generation). For sentiment classification, we test on SST-2~\citep{scarlatos2023reticl}, IMDB~\citep{maas2011learning} and Amazon (polarity)~\citep{mcauley2013hidden}. For commonsense reasoning, we use ARC-Easy~\citep{clark2018think} and CommonsenseQA (CsQA)~\citep{talmor2019commonsenseqa}. For text to SQL generation, we use GeoQuery~\citep{zelle1996learning,tang2001using}. For math problems, we test on GSM8K~\citep{cobbe2021gsm8k} and GSM-Plus-Mini~\citep{DBLP:conf/acl/LiCZKB24} datasets. For reading comprehension, we use SQuAD~\citep{rajpurka2016squad} and SCIQ~\citep{welbl2017crowdsourcing} datasets. 
We subsample some datasets to reduce the computation resources needed.


\textbf{Models.} Our main experiments are conducted on Llama 3.1 and Llama 3.2~\citep{dubey2024llama}, Gemma 2~\citep{team2024gemma}, and Mistral v0.3~\citep{jiang2023mistral} families of models. For math problems (GSM8K and GSM-Plus-Mini), we use the instruction-tuned LLMs, while for all other datasets, we use the base model. For the main experiments, we use Sentence-BERT~\citep{reimers2019sentence} to compute all the embeddings for \topk, \diversity, and \topkdiv. 
Experiments are conducted on 2 A100 GPUs.


\textbf{Hyperparameters.} For \diversity, we choose to first reduce the demonstration set $D$ to a ``coreset'' $D_r$ with size $100$. This choice balances full similarity selection (\topk) and methods focusing mainly  on diversity. 
For \topkdiv, we choose $\alpha=1/2$ to balance between similarity and diversity. For the classification task, we predict positive if the logit for token \texttt{great} is larger than that for token \texttt{terrible} for the next token prediction given a prompt. For multi-choice tasks, we choose the option with the lowest average CE loss given a prompt. For generation tasks (text to SQL, math, reading comprehension), we use greedy decoding. 
More experimental details, including the prompt for each task, can be found in \Cref{sec:more-exp-details}.

\subsection{Main findings}\label{sec:main-findings}

\begin{table*}[!t]
\centering
\small
\caption{\textbf{(Comparison of different in-context example selection methods)} We compare diversity-aware methods \diversity and \topkdiv  with randomly chosen (\rand) and similarity-based method \topk on a variety of tasks using different models with different number of in-context examples $K$. For \topk and \topkdiv, we test ten different permutations of the demonstration due to the determined choice by these methods; For \rand and \diversity, we test ten different random seeds. We use the corresponding instruct-tuned model for math tasks (GSM8K and GSM-Plus-Mini) and base model for all other tasks.  For TopK and TopK-Div methods - both being deterministic approaches - we computed outcomes across ten distinct example permutations. For Rand and Div methods, we report the averaged results across ten random seeds. There is a huge improvement when the shot number increases from 0 to 4 / 8,  which demonstrates the effectiveness of our example selection. Due to the absence of prior knowledge for Geoquery in the zero-shot ($k=0$) setting, we omit its $k=0$ results. The bold entries indicate optimal performances. The $\mathrm{std}$ is no more than 1\% in most cases; see Appendix~\ref{sec:hyperparameter-ablation} for details.}
\label{tab:main-res}
\begin{adjustbox}{width=0.92\textwidth}
\begin{tabular}{cccccccccccc}
\toprule        
\multirow{2}{*}{Model} & \multicolumn{1}{c}{\multirow{2}{*}{$K$}} & \multirow{2}{*}{Method}  & \multicolumn{2}{c}{Classification} & \multicolumn{2}{c}{Multi-choice}  & \multicolumn{2}{c}{Math} & Code & \multicolumn{2}{c}{Reading}\\
& & & SST-2 & Amazon & ARC-Easy & CsQA & GSM8K & GSM-Plus-Mini & GeoQuery & SQuAD & SCIQ \\
\midrule
\parbox[t]{3mm}{\multirow{8}{*}{\rotatebox[origin=c]{90}{Llama-3.1-8B}}} 
& \multirow{1}{*}{0} 
& - & $87.50$ & $95.40$ & $82.43$ & $62.80$ & $53.45$ & $65.12$ & --- & $42.30$& $36.40$\\
\cline{2-12}
& \multirow{4}{*}{4} 
& \rand & $91.31$&$96.38$ & $84.72$& $71.15$& $\mathbf{82.24}$& $66.90$ &  $12.50$& $\mathbf{75.95}$& $74.00$\\
& & \topk &  $\mathbf{94.13}$&$96.24$&   $\mathbf{86.10}$& $72.54$& $81.99$ & $65.30$ & $62.79$& $73.51$& $72.70$\\
& & \diversity & $91.50$&$96.18$ & $85.06$&  $71.17$&  $82.14$&  $\mathbf{66.92}$ &  $33.79$& $75.66$& $\mathbf{74.47}$\\
& & \topkdiv& $92.75$ & $\mathbf{96.43}$&  $85.83$& $\mathbf{72.57}$& $81.74$& $66.12$ &  $\mathbf{71.14}$&$73.28$ & $73.87$\\
\cline{2-12}
& \multirow{4}{*}{8} 
& \rand & $92.27$ &$\mathbf{96.63}$ & $84.38$& $72.23$& $82.81$ & $\mathbf{66.72}$ & $23.11$& $77.13$& $74.65$\\
& & \topk & $\mathbf{93.64}$&$96.12$ & $\mathbf{85.91}$& $\mathbf{73.91}$& $82.26$ & $65.99$ & $72.04$&  $75.52$& $74.72$\\
& & \diversity & $92.95$ & $96.25$ & $84.97$& $72.77$& $\mathbf{82.98}$ & $66.56$ & $38.61$& $\mathbf{77.71}$& $\mathbf{75.17}$ \\
& & \topkdiv & $93.33$ & $96.57$ & $85.39$ & $73.76$ & $82.63$ & $66.48$ & $\mathbf{78.68}$ & $76.13$ & $75.07$\\
\midrule
\parbox[t]{3mm}{\multirow{8}{*}{\rotatebox[origin=c]{90}{Gemma-2-9B}}} 
& \multirow{1}{*}{0} 
& - & $67.50$ & $85.10$ & $88.15$ & $61.80$ & $16.07$& $32.79$ & ---& $37.90$& $41.10$\\
\cline{2-12}
& \multirow{4}{*}{4} 
& \rand & $93.33$& $96.15$& $89.52$& $74.70$& $84.29$ & $74.40$ & $13.89$& $\mathbf{77.19}$& $75.80$\\
& & \topk & $\mathbf{94.47}$& $96.34$& $\mathbf{90.50}$& $75.19$& $84.25$& $\mathbf{74.50}$ & $61.14$& $74.82$& $75.24$\\
& & \diversity & $93.45$& $95.69$& $90.03$& $74.85$& $\mathbf{84.44}$& $73.34$ &$36.29$ & $77.06$& $\mathbf{75.96}$\\
& & \topkdiv& $93.34$& $\mathbf{96.57}$& $90.19$& $\mathbf{75.60}$& $83.54$& $74.47$ & $\mathbf{70.43}$& $75.05$& $75.21$\\
\cline{2-12}
& \multirow{4}{*}{8} 
& \rand & $93.30$& $96.09$& $89.39$& $75.98$& $\mathbf{84.34}$& $74.48$ & $24.36$& $\mathbf{79.23}$& $76.28$\\
& & \topk & $\mathbf{94.20}$& $96.55$& $\mathbf{90.62}$& $76.14$& $83.57$& $\mathbf{75.36}$ & $71.00$& $77.59$& $75.55$\\
& & \diversity & $93.41$& $95.94$& $89.90$& $\mathbf{76.60}$& $84.22$& $74.69$ & $42.07$& $79.05$& $\mathbf{76.65}$\\
& & \topkdiv & $94.04$& $\mathbf{96.58}$& $90.48$& $76.53$ & $83.85$& $75.16$ &$\mathbf{76.32}$ & $77.64$& $76.24$\\
\midrule
\parbox[t]{3mm}{\multirow{8}{*}{\rotatebox[origin=c]{90}{Mistral-7B-v0.3}}} 
& \multirow{1}{*}{0} 
& - & $66.50$ & $94.00$ & $76.41$ & $51.80$& $9.48$ & $5.17$ & --- &$30.50$ &$34.20$ \\
\cline{2-12}
&  \multirow{4}{*}{4} 
& \rand & $91.00$& $94.02$& $82.77$ & $69.83$ & $48.78$ &  $37.20$ & $12.14$&  $\mathbf{76.70}$ & $74.71$\\
& & \topk & $\mathbf{93.57}$& $\mathbf{96.17}$& $\mathbf{85.21}$ & $69.73$ & $49.28$ & $38.20$ & $60.14$&  $75.04$ & $73.73$\\
& & \diversity & $91.98$& $94.15$& $82.98$ & $\mathbf{70.15}$ & $49.49$ & $37.50$ &  $34.89$& $75.96$ & $\mathbf{75.83}$ \\
& & \topkdiv& $92.73$& $95.90$& $84.55$ & $69.91$ & $\mathbf{49.99}$ & $\mathbf{38.45}$ &$\mathbf{71.46}$ &   $74.43$& $73.16$\\
\cline{2-12}
& \multirow{4}{*}{8} 
& \rand & $92.49$& $95.35$& $83.69$ & $71.65$ &$47.86$ & $36.32$ & $22.18$& $77.30$& $75.54$\\
& & \topk & $\mathbf{93.61}$& $\mathbf{96.15}$& $\mathbf{85.17}$ &  $71.88$& $48.43$ & $37.35$ & $70.50$& $77.05$ &  $75.44$\\
& & \diversity & $92.55$ & $95.10$ & $84.27$ & $\mathbf{72.04}$ & $48.33$ & $36.12$ & $39.14$ & $\mathbf{77.67}$  & $\mathbf{76.30}$\\
& & \topkdiv & $93.47$ & $96.11$ & $84.85$ & $71.81$ & $\mathbf{48.60}$ & $\mathbf{37.81}$ & $\mathbf{77.93}$ &   $77.44$ & $75.22$\\
\bottomrule
\end{tabular}
\end{adjustbox}
\end{table*}

\textbf{Diversity-aware methods perform better on more ``challenging'' tasks.}
\Cref{tab:main-res} summarizes our main results in the in-distribution (ID) setting, where the demonstration distribution $\gD_\gX$ matches the query distribution $\gQ_\gX$. 
For simpler tasks like sentiment classification, \topk consistently performs best, significantly outperforming diversity-emphasizing methods (\rand and \diversity), while \topkdiv (balancing similarity and diversity) typically ranks between these extremes. 
In commonsense reasoning (multi-choice), introducing some diversity via \topkdiv improves performance over pure \topk, as observed in Commonsense QA, although the gains on ARC-Easy remain modest.

For more complex tasks—including reading comprehension, text-to-SQL generation, multi-step mathematical reasoning, and GeoQuery—introducing diversity consistently improves performance over the similarity-based \topk baseline. 
In GeoQuery specifically, diversity (\topkdiv) yields at least a 7\% absolute accuracy gain, likely due to enhanced feature coverage~\citep{levy2023diverse,ye2023complementary}; however, excessive diversity (\rand and \diversity) becomes detrimental, as overly dissimilar examples fail to illustrate coherent solution patterns. 
For math and reading comprehension tasks, methods emphasizing diversity (\diversity and even \rand) also surpass \topk. 
Notably, the success of random selection here cannot be attributed purely to coverage, as random demonstrations do not systematically capture similar structures. 
Instead, we hypothesize that when models already possess strong zero-shot capabilities (as indicated by tasks like Reading and Math), diversity helps by directing the model toward fundamental reasoning skills rather than memorization.

\begin{table*}[!t]
    \centering
    \small
    \caption{Performance of 0-shot and 1-shot Baseline in Code and Reading Tasks. When $k=1$, there is only one possible permutation, so we report a single result for both TopK and TopK-Div methods. For Rand and Div approaches, we report the averaged results across ten random seeds. Embedding = all-roberta-large-v1.}
    \label{tab:0-shot}
    \begin{adjustbox}{width=0.92\textwidth}
    \begin{tabular}{ccccccccccc}
        \toprule
        \multirow{2}{*}{Model} & \multirow{2}{*}{Dataset} & \multicolumn{1}{c}{$K=0$} & \multicolumn{4}{c}{$K=1$} & \multicolumn{4}{c}{$K=4$} \\
        \cmidrule(lr){3-3} \cmidrule(lr){4-7} \cmidrule(lr){8-11}
        & & - & Rand & Topk & Div & Topk-Div & Rand & Topk & Div & Topk-Div \\
        \midrule
        \parbox[t]{25mm}{\multirow{2}{*}{{Llama-3.1-8B}}} 
        & Code (Geoquery) & $-$ & $2.61$ & $37.14$ & $16.93$ & $37.14$ & $12.57$ & $63.04$ & $33.71$ & $71.07$ \\
        & Reading (SQuAD) & $42.30$ & $68.64$ & $67.00$ & $67.87$ & $67.00$ & $75.95$ & $73.51$ & $75.66$ & $73.28$ \\
        \midrule
        \parbox[t]{25mm}{\multirow{2}{*}{{Gemma-2-9B}}} 
        & Code (Geoquery) & $-$ & $3.07$ & $41.43$ & $16.71$ & $41.43$ & $13.89$ & $61.14$ & $36.29$ & $70.43$ \\
        & Reading (SQuAD) & $37.90$ & $71.34$ & $69.00$ & $70.69$ & $69.00$ & $77.19$ & $74.82$ & $77.06$ & $75.05$ \\
        \midrule
        \parbox[t]{25mm}{\multirow{2}{*}{{Mistral-7B-v0.3}}} 
        & Code (Geoquery) & $-$ & $2.75$ & $40.71$ & $18.39$ & $40.71$ & $12.14$ & $60.14$ & $34.89$ & $71.46$ \\
        & Reading (SQuAD) & $30.50$ & $69.12$ & $66.30$ & $67.80$ & $66.30$ & $76.70$ & $75.04$ & $75.96$ & $74.43$ \\
        \bottomrule
    \end{tabular}
    \end{adjustbox}
\end{table*}

To verify whether the model inherently possesses the ability to solve certain tasks, we tested its 0-shot and 1-shot performance on the SQuAD and GeoQuery datasets. For the Reading task, accuracy is calculated only when the output exactly matches the answer, imposing strict format requirements. Consequently, on SQuAD, once the model understood the output format in the 1-shot setting, the absolute performance gap compared to the 4-shot setting was less than 8\%. However, on GeoQuery, even after the model grasped the output format via the 1-shot example, the absolute performance gap compared to the 4-shot setting was still over 20\%.
Therefore, the model possesses a strong inherent ability to solve the Reading task (a similar conclusion also holds for the Math task). Conversely, the model itself lacks domain-specific knowledge related to GeoQuery and thus needs to learn more from the provided context.

Given the clear benefits of diversity for more challenging tasks, we conjecture that increased task complexity demands greater diversity within \topkdiv (lowering the diversity parameter $\alpha$). 
Let $\mathrm{Acc}_{\alpha}$ denotes the accuracy of \topkdiv parameterized by $\alpha$ (\Cref{eqa:topkdiv}), We define:
\begin{equation} \label{eqa:topkdiv-difference}
\textstyle \Delta =\frac{1}{5}\sum\limits_{i=6}^{10} \mathrm{Acc_{i/10}} - \frac{1}{5}\sum\limits_{i=1}^{5} \mathrm{Acc_{i/10}}.
\end{equation} 
The difference $\Delta$ quantifies the gap between the average accuracy of lower diversity and higher diversity (e.g, $\Delta > 0$ means less diversity is better). 
\Cref{tab:best knn diversity} confirms this hypothesis: Minimal diversity is optimal for simpler tasks, while higher diversity consistently enhances performance as task difficulty increases.
\begin{table*}[!t]
    \centering
    \small
    \caption{\textbf{Comparison of \topkdiv results with different $\alpha$}. We report $\Delta$ (\Cref{eqa:topkdiv-difference}) for Classification, Multi-choice and Reading task accross six datasets. For each value of $\alpha$ in \topkdiv, we tested ten permutations and calculated the mean. We also report the average $\Delta$ for both shots and both models. On the simplest Classification task (SST-2, Amazon), the model performs significantly better with \topkdiv when less diversity is added. On the most difficult Reading task (SCIQ, SQuAD), the model performs significantly better with \topkdiv when more diversity is added. For the Multi-choice task (ARC-Easy, CSQA), which has a difficulty level between the two, the difference is not significant.}
    \label{tab:best knn diversity}
    \begin{tabular}{cccccccc}
    \toprule
    \multirow{2}{*}{Model} & \multirow{2}{*}{$K$} & \multicolumn{6}{c}{$\Delta$} \\
    & & SST-2 & Amazon & ARC-Easy & CSQA & SCIQ & SQUAD \\
    \midrule
    \multirow{2}{*}{Llama-3.2-3B}
    & 4 & 0.11\% & 0.12\% & 0.45\% & 0.41\% & -0.80\% & 0.40\% \\
    \cmidrule(lr){2-8}
    & 8 & 1.05\% & -0.01\% & 0.20\% & 0.06\% & -0.74\% & -1.01\% \\
    \midrule
    \multirow{2}{*}{Gemma-2-2B}
    & 4 & 0.34\% & 0.50\% & 0.87\% & -0.37\% & -0.07\% & 0.04\% \\
    \cmidrule(lr){2-8}
    & 8 & 0.19\% & 0.27\% & -0.15\% & 0.18\% & -0.49\% & -0.82\% \\
    \midrule
    \multirow{1}{*}{Average}
    & - & 0.42\% & 0.22\% & 0.34\% & 0.07\% & -0.53\% & -0.35\% \\
    \bottomrule
    \end{tabular}
 \end{table*}

\begin{table}[!t]
\centering
\small
\caption{\textbf{(Comparison of different methods when demonstration and query come from different distribution)} We compare the methods on different tasks. The number of shots is fixed as $K=4$. We observe that diversity-aware methods are more robust to out-of-distribution query. The performance drop from ID to OOD on \topk is in general larger than diversity-aware methods.}
\label{tab:ood-res}
\begin{tabular}{ccccccc}
\toprule
& Test & Demo. & \rand & \topk & \diversity & \topkdiv \\
\midrule
\parbox[t]{3mm}{\multirow{7}{*}{\rotatebox[origin=c]{90}{Llama-3.1-8B}}} &\multirow{3}{*}{SST-2} & SST-2 & 91.35& \textbf{93.90}& 91.88& 92.40  \\
& & \hlcella IMDB & \hlcella 88.85& \hlcella \textbf{90.80}& \hlcella 90.71& \hlcella \textbf{90.80}\\
& & \hlcella Amazon & \hlcella 88.28& \hlcella 89.50& \hlcella 86.64& \hlcella \textbf{89.60}\\
\cline{2-7}
& \multirow{2}{*}{CsQA} & CsQA & 70.93& 72.40& 70.95& \textbf{72.80}\\
&  & \hlcella ARC-Easy &\hlcella 66.86&\hlcella 66.70&\hlcella 67.08&\hlcella \textbf{67.70}\\
\cline{2-7}
& \multirow{2}{*}{SCIQ} & SCIQ & 74.17& 72.80& \textbf{74.44}& 73.60\\
&  & \hlcella SQuAD &\hlcella 72.11&\hlcella 71.40&\hlcella \textbf{72.79}&\hlcella 71.60\\
\midrule
\parbox[t]{3mm}{\multirow{7}{*}{\rotatebox[origin=c]{90}{Gemma-2-9B}}} & \multirow{3}{*}{SST-2} & SST-2 & 93.23& \textbf{93.80}& 93.43& 93.40  \\
&  & \hlcella IMDB & \hlcella 88.66& \hlcella 89.90& \hlcella 88.59& \hlcella \textbf{91.10}\\
&  & \hlcella Amazon & \hlcella 88.69& \hlcella 89.40& \hlcella \textbf{90.49}& \hlcella 89.60\\
\cline{2-7}
& \multirow{2}{*}{CsQA} & CsQA & 74.57& 75.00& 74.61& \textbf{75.50}\\
&  & \hlcella ARC-Easy &\hlcella 68.30&\hlcella 68.90&\hlcella 68.58&\hlcella \textbf{69.50}\\
\cline{2-7}
& \multirow{2}{*}{SCIQ} & SCIQ &  75.60&  75.50&  \textbf{76.48}&  74.80\\
&  & \hlcella SQuAD &\hlcella  73.63&\hlcella  73.60&\hlcella \textbf{74.64} &\hlcella 73.50 \\
\bottomrule
\end{tabular}
\end{table}

\begin{figure}[!t]
    \centering
    \includegraphics[width=0.85\linewidth]{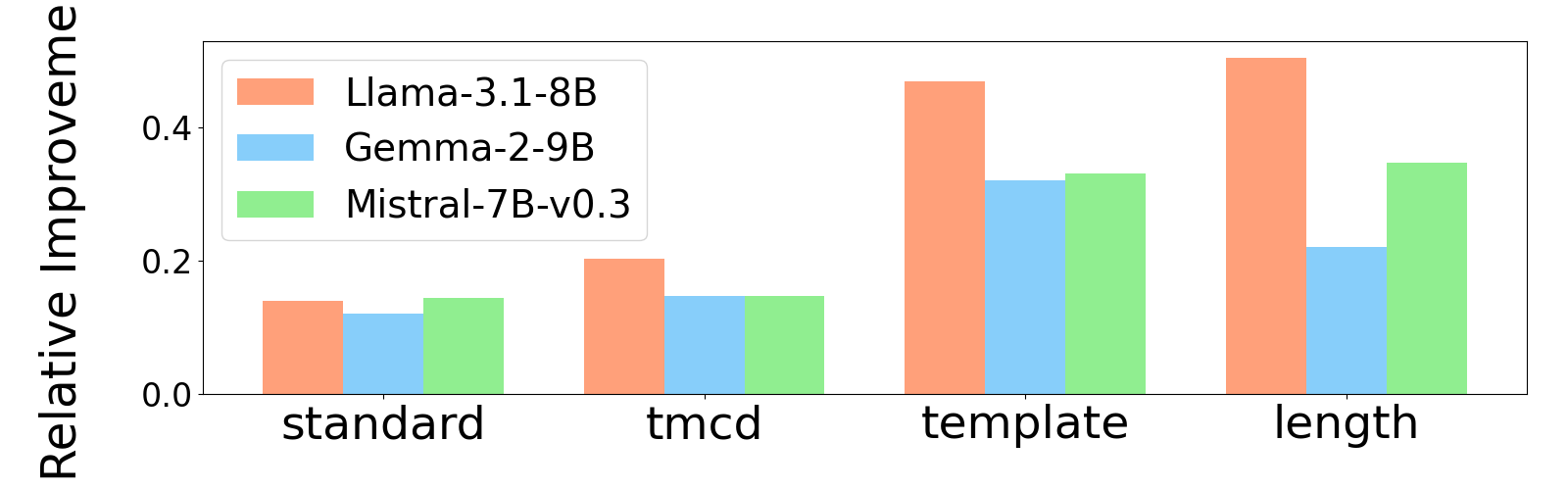}
    \caption{\small \textbf{(Comparison of different methods on GeoQeury OOD setting)} We report the relative improvement of \topkdiv over \topk when demonstrations and queries come from different GeoQuery dataset splitting ways. ``standard'' split denotes ID the setting. The relative improvement enlarges in the OOD setting.}
    \label{fig:geoquery-splits}
\end{figure}

\textbf{Diversity helps out-of-distribution generalization.}
\Cref{tab:ood-res} presents results on sentiment classification, commonsense reasoning, and reading comprehension, while \Cref{fig:geoquery-splits} shows text-to-SQL generation performance in the out-of-distribution (OOD) setting, where the demonstration distribution $\gD_\gX$ and query distribution $\gQ_\gX$ differ.

Overall, diversity improves OOD in-context learning. In sentiment classification, \topk performs best when both demonstrations and queries come from SST-2. 
However, when demonstrations shift to IMDB (another movie review dataset), \topk and \topkdiv perform similarly. When using Amazon (a shopping review dataset) as demonstrations, \topkdiv surpasses \topk.
A similar trend is observed in commonsense reasoning: replacing Commonsense QA (ID) demonstrations with ARC-Easy (OOD) increases the performance gap between \diversity and \topk from 0.4\% to 1.0\%. 
Text-to-SQL generation follows this pattern, with a larger improvement in OOD settings. 
Additionally, we note that GSM-Plus-Mini serves as an OOD setting for GSM8K (Math in Table \ref{tab:main-res}), as they share the same training set. 
A larger improvement from adding diversity is also observed on GSM-Plus-Mini.

\begin{table}[!t]
\centering
\caption{\textbf{(Comparison of different methods on math when demonstration and query come from different distribution)} OOD setting for math problem where the test dataset is chosen to be PRM800K. We find that,  diversity-aware methods are more superior than \topk in OOD setting. The method that achieves the best in each setting is highlighted. We also report the average results of each method for both shots and all models.}
\label{tab:ood-math}

\begin{tabular}{ccccccc}
\toprule
Model & Shots & Demo. & \rand & \topk & \diversity & \topkdiv \\
\midrule
\multirow{4}{*}{Llama-3.1-8B} &\multirow{2}{*}{$K=4$} & PRM800K & 43.50 & 41.40 & \textbf{44.86} &  44.80  \\
& & \hlcella GSM8K & \hlcella 41.50 & \hlcella 41.00 & \hlcella \textbf{43.28} & \hlcella 42.00 \\
 &\multirow{2}{*}{$K=8$} & PRM800K & 43.32 & 43.00 & \textbf{44.28} & 40.00 \\
& & \hlcella GSM8K & \hlcella 41.66 & \hlcella 42.00 & \hlcella \textbf{43.46} & \hlcella 40.80 \\
\midrule
\multirow{4}{*}{Llama-3.1-70B} & \multirow{2}{*}{$K=4$} & PRM800K & 57.78 & 58.20 & 57.42 & \textbf{59.40} \\
& & \hlcella GSM8K & \hlcella 60.62 & \hlcella  \textbf{62.00} & \hlcella 61.88 & \hlcella 61.00 \\
 &\multirow{2}{*}{$K=8$} & PRM800K & 54.72 & \textbf{59.00} & 55.86 & 58.00 \\
& & \hlcella GSM8K & \hlcella \textbf{61.14} & \hlcella 59.60 & \hlcella 60.96 & \hlcella 60.00  \\
\midrule
\multirow{4}{*}{Gemma-2-9B} &\multirow{2}{*}{$K=4$} & PRM800K & 38.04 & 42.40 & 36.78 & \textbf{44.40} \\
& & \hlcella GSM8K & \hlcella \textbf{42.10} & \hlcella 41.00 & \hlcella 41.04 & \hlcella 42.20 \\
 &\multirow{2}{*}{$K=8$} & PRM800K & 40.66 & \textbf{46.20} & 39.20 & 44.40 \\
& & \hlcella GSM8K & \hlcella 42.06 & \hlcella 41.80 & \hlcella \textbf{42.74} & \hlcella 41.60 \\
\midrule
\multirow{4}{*}{Gemma-2-27B} &\multirow{2}{*}{$K=4$} & PRM800K & 46.06 & 49.20 & 47.80 & \textbf{49.60} \\
& & \hlcella GSM8K & \hlcella 46.06 & \hlcella 46.00 & \hlcella \textbf{46.30} & \hlcella 45.20 \\
 &\multirow{2}{*}{$K=8$} & PRM800K & 47.10 & \textbf{50.40}  & 47.04 & 48.60 \\
& & \hlcella GSM8K & \hlcella 45.40 & \hlcella 44.80 & \hlcella \textbf{45.92} & \hlcella 45.20 \\
\midrule
\multirow{2}{*}{Average} & \multirow{2}{*}{-} 
& PRM800K & 61.86 & \textbf{64.97} & 62.21 & 64.87 \\
& & \hlcella GSM8K & \hlcella 63.42 & \hlcella  63.03 & \hlcella \textbf{64.26} & \hlcella 63.00 \\
\bottomrule
\end{tabular}
\end{table}

To further verify the greater benefits of adding diversity in more challenging OOD scenarios, we designed experiments on math tasks with an increased degree of distribution shift.
We use GSM8K as the demonstration set and PRM800K~\citep{lightman2023lets} as the query set, establishing a significant OOD setting. We show the OOD results on Llama-3.1-8B/70B and Gemma-2-9B/27B instruct-tuned models. \Cref{tab:ood-math} summarizes our result. We observe that under this PRM800K OOD setting, methods incorporating diversity consistently demonstrate stronger robustness to this distribution shift. Even for Gemma models where \topk performs very well on ID tasks (demonstration and query set are all PRM800K), \topk is outperformed by diversity-aware methods on the OOD setting. A similar trend also holds for Llama models. One interesting finding is that for PRM800K, more demonstration might not lead to better performance, and also in our experiment, using GSM8K as demonstration works better than using PRM800K data as demonstrations.

For reading comprehension, switching to an OOD demonstration dataset does not significantly widen the gap between \diversity and \topk, but \diversity still outperforms \topk.
%



\begin{table}[!t]
    \centering
    \small
    \caption{Relative improvement of \topkdiv over \topk on GeoQuery and SQuAD on different sets of the queries.  For the GeoQuery dataset, we fine-tuned both base models on its training set. We categorized questions in testing set as ``Easy'' if the fine-tuned models correctly answered them in a 0-shot setting, and as ``Hard'' if these models failed to answer them correctly in the same 0-shot setting. We report the performance of both methods in a 4-shot setting. For SQuAD, we split the testing set only using the fine-tuned gemma-2-9B model, since fine-tuning the Llama-3.1-8B model yielded poor results. We observe that \topkdiv exhibits greater improvement on ``Hard'' examples.
}
    \label{tab:harder-example} 
    \begin{tabular}{llcccc} 
    \toprule
    \multirow{2}{*}{Split} & \multirow{2}{*}{Method} & \multicolumn{2}{c}{Gemma-2-9B} & \multicolumn{2}{c}{Lamma-3.1-8B} \\
    & & GeoQuery & SQuAD & GeoQuery & SQuAD \\
    \midrule
    \multirow{2}{*}{Easy} 
    & \topk & 72.09 &  83.01 & 79.31 &  81.04    \\
    & \topkdiv  & 77.91  &  82.66 & 83.71 &  79.65    \\
    \cline{1-6} 
    \multirow{2}{*}{Hard} 
    & \topk & 56.29 &  20.00 & 51.52 &  24.44    \\
    & \topkdiv  & 67.11 & 23.70  & 62.13 & 25.19     \\
    \bottomrule
    \end{tabular}
\end{table}

\textbf{Diversity performs better on harder examples.} 
Besides discussing the performance of diversity-aware methods (\topkdiv, \diversity, and even \rand) at task levels, we also analyze which specific examples benefit most from diversity.
For this, we first need to quantify the ``difficulty level'' of examples. 
Motivated by~\citep{DBLP:conf/emnlp/SwayamdiptaSLWH20}, we use whether a model can correctly answer a question after fine-tuning as an indicator of that question's difficulty for a specific language model. 
Therefore, we fine-tuned the corresponding base model on the dataset's training set using LoRA. 
Subsequently, based on whether this fine-tuned model could accurately answer questions in the testing set under a zero-shot setting, we classified these questions as ``easy'' or ``hard''.

We examine this phenomenon in \textsc{GeoQuery} and \textsc{SQuAD}, where \topkdiv consistently outperforms \topk. 
\Cref{tab:harder-example} shows that diversity yields greater benefits on harder examples. 
In \textsc{GeoQuery}, the absolute accuracy improvement of \topkdiv over \topk is 5.11\% on easy examples (averaged across two models), increasing to 10.72\% on hard examples. 
In \textsc{SQuAD}, while \topkdiv slightly underperforms \topk on easy examples by 0.87\%, it outperforms \topk on hard examples by 2.23\%.

\textbf{Diversity v.s. coverage.}
We analyze how diversity relates to coverage at both the example and task levels. 

At the example level, the phenomenon of ``diversity performs better on harder examples'' observed in GeoQuery (\Cref{tab:harder-example}) can be explained by enhanced coverage: harder examples often require covering more local structures~\citep{levy2023diverse,gupta2023coverage}, thus benefiting diversity-driven methods. 
However, in the case of SQuAD, this phenomenon likely arises from a different mechanism: when $k = 1$, \topk still underperforms \rand  /  \diversity methods. We speculate this is because the support in the original dataset contains a lot of noise, causing similar examples not only to fail to provide effective information but also potentially to mislead the model into focusing on noisy information (``coverage'' isn't helpful in such case).
%

\begin{table}[!t]
    \centering
    \small
    \caption{Results for SQuAD with cut perturbation. We performed content trimming on the support portion of the SQuAD dataset using Deepseek-r1, retaining only the top 1/3 most answer-relevant content. SQuAD-Cut refers to trimming applied solely to the testing set, while SQuAD-Both-Cut indicates trimming performed on both testing and training sets. The values in parentheses represent performance improvements relative to the original SQuAD dataset.}
    \label{tab:reading task squad perturbation}
    \begin{tabular}{ccccccc}
    \toprule
    \multirow{2}{*}{Model} & \multirow{2}{*}{$K$} & \multirow{2}{*}{Dataset} & \multicolumn{4}{c}{Method} \\
    & & & Rand & Topk & Div & Topk-Div \\
    \midrule
    \multicolumn{1}{c}{\multirow{9}{*}{\rotatebox[origin=c]{90}{Llama-3.1-8B}}} 
    & \multirow{3}{*}{1} 
    & SQuAD & 68.64 & 67.00 & 67.87 & 67.00 \\
    & & SQuAD-Cut & 69.43 (+0.79) & \textbf{68.20 (+1.20)} & 68.45 (+0.58) & 67.70 (+0.70) \\
    & & SQuAD-Both-Cut & 69.71 (+1.07) & \textbf{69.90 (+2.90)} & 69.47 (+1.60) & \textbf{69.90 (+2.90)} \\
    \cline{2-7}
    & \multirow{3}{*}{4} 
    & SQuAD & 75.95 & 73.51 & 75.66 & 73.28 \\
    & & SQuAD-Cut & 77.15 (+1.2) & 75.96 (+2.45) & 77.00 (+1.34) & \textbf{76.89 (+2.61)} \\
    & & SQuAD-Both-Cut & 76.95 (+1.00) & 76.15 (+2.64) & 77.76 (+2.10) & \textbf{76.47 (+3.19)} \\
    \cline{2-7}
    & \multirow{3}{*}{8} 
    & SQuAD & 77.13 & 75.52 & 77.71 & 76.13 \\
    & & SQuAD-Cut & 79.10 (+1.97) & 77.43 (+1.91) & \textbf{79.43 (+1.72)} & 78.64 (+2.51) \\
    & & SQuAD-Both-Cut & 79.39 (+2.26) & \textbf{78.66 (+3.14)} & 79.26 (+1.55) & 79.13 (+3.00) \\
    \midrule
    \multicolumn{1}{c}{\multirow{9}{*}{\rotatebox[origin=c]{90}{Gemma-2-9B}}} 
    & \multirow{3}{*}{1} 
    & SQuAD & 71.34 & 69.00 & 70.69 & 69.00 \\
    & & SQuAD-Cut & 72.96 (+1.62) & \textbf{71.20 (+2.20)} & 72.25 (+1.56) & 71.10 (+2.10) \\
    & & SQuAD-Both-Cut & 73.14 (+1.80) & \textbf{72.30 (+3.30)} & 72.67 (+1.98) & \textbf{72.30 (+3.30)} \\
    \cline{2-7}
    & \multirow{3}{*}{4} 
    & SQuAD & 77.19 & 74.82 & 77.06 & 75.05 \\
    & & SQuAD-Cut & 78.75 (+1.56) & 77.64 (+2.82) & 78.23 (+1.17) & \textbf{78.65 (+3.60)} \\
    & & SQuAD-Both-Cut & 78.72 (+1.53) & \textbf{77.47 (+2.65)} & 78.54 (+1.48) & 76.78 (+1.73) \\
    \cline{2-7}
    & \multirow{3}{*}{8} 
    & SQuAD & 79.23 & 77.59 & 79.05 & 77.64 \\
    & & SQuAD-Cut & 80.41 (+1.18) & 79.74 (+2.15) & 80.45 (+1.40) & \textbf{80.72 (+3.08)} \\
    & & SQuAD-Both-Cut & 80.22 (+0.99) & \textbf{79.05 (+1.46)} & 80.10 (+1.05) & 78.84 (+1.20) \\
    \midrule
    \multicolumn{1}{c}{\multirow{9}{*}{\rotatebox[origin=c]{90}{Mistral-7B-v0.3}}} 
    & \multirow{3}{*}{1} 
    & SQuAD & 69.12 & 66.30 & 67.80 & 66.30 \\
    & & SQuAD-Cut & 71.38 (+2.26) & \textbf{69.70 (+3.40)} & 69.21 (+1.41) & \textbf{69.70 (+3.40)} \\
    & & SQuAD-Both-Cut & 71.44 (+2.32) & \textbf{70.70 (+4.40)} & 72.00 (+4.20) & \textbf{70.70 (+4.40)} \\
    \cline{2-7}
    & \multirow{3}{*}{4} 
    & SQuAD & 76.70 & 75.04 & 75.96 & 74.43 \\
    & & SQuAD-Cut & 77.78 (+1.08) & 77.78 (+2.74) & 77.18 (+1.22) & \textbf{78.70 (+4.37)} \\
    & & SQuAD-Both-Cut & 77.76 (+1.06) & 77.92 (+2.88) & 77.89 (+1.93) & \textbf{77.40 (+2.97)} \\
    \cline{2-7}
    & \multirow{3}{*}{8} 
    & SQuAD & 77.30 & 77.05 & 77.67 & 77.44 \\
    & & SQuAD-Cut-R1 & 79.00 (+1.70) & \textbf{79.09 (+2.04)} & 78.64 (+0.97) & 79.07 (+1.63) \\
    & & SQuAD-Both-Cut-R1 & \textbf{79.92 (+2.62)} & 78.76 (+1.71) & 79.61 (+1.94) & 79.64 (+2.20) \\
    \bottomrule
    \end{tabular}
    \end{table}

To illustrate this ``beyond coverage'' effect, we conducted an experiment removing irrelevant noise from the SQuAD dataset.
Using DeepSeek-R1, we removed information irrelevant to the answer from the support passages in SQuAD, reducing content by approximately 50\%. Based on this, we constructed two variants: SQuAD-Cut, where only the training set is streamlined, and SQuAD-Both-Cut, where both the training and test sets are streamlined. As shown in \Cref{tab:reading task squad perturbation}, the more streamlined (i.e., higher-quality and less noisy) the dataset, the better the performance of \topk and \topkdiv. Notably, their improvement margins are significantly larger than that of \diversity (though still more than 1\% lower than \diversity). This indicates that when the dataset quality is higher, the ``Coverage'' mechanism can focus on high signal-to-noise ratio information (rather than incorrectly covering noise), and its effectiveness is significantly enhanced. \topk-based methods are more likely to “cover” high-quality information segments truly relevant to the answer, whereas \diversity, as an intrinsic metric, inherently includes effective mechanisms not directly dependent on precise semantic coverage (e.g., structural diversity: selecting examples with different sentence structures or argumentation styles). These mechanisms already play a role in the original noisy data, avoiding overfitting to noise, causing it to outperform noise-sensitive coverage strategies, and its baseline performance is already relatively robust. This fully demonstrates that the value of \emph{diversity} is ``beyond coverage''.

At the task level, as shown in \Cref{tab:best knn diversity}, increasing the number of shots (from 4 to 8) enhances the benefits of diversity. 
While coverage-based explanations focus on capturing local features or knowledge, this phenomenon suggests a ``beyond coverage'' benefit of diversity: when provided with more demonstrations, the model leverages diversified examples to better reconstruct and generalize the broader conceptual theme underlying the task. 
Intuitively, fewer shots limit the model's ability to form such a coherent conceptual understanding, making simple similarity or coverage sufficient. 
However, as the number of examples increases, adding diversity enables the model to abstract and synthesize the core task concepts more effectively, thereby enhancing performance beyond what mere coverage can explain.
In \Cref{sec:theory}, we provide theoretical justifications that further clarify why diversity can enhance performance through a mechanism that goes beyond feature coverage.

\subsection{Ablation studies}\label{sec:ablation}

We conduct experiments to observe how the improvement of diversity-aware methods over \topk changes if we change the size of the LLM used, since it is possible that, as the models scale up, ICL is not sensitive to data selection, and thus the improvement of diversity over pure similarity diminishes.

\begin{figure}[!t]
\centering
    \includegraphics[width=0.9\linewidth]{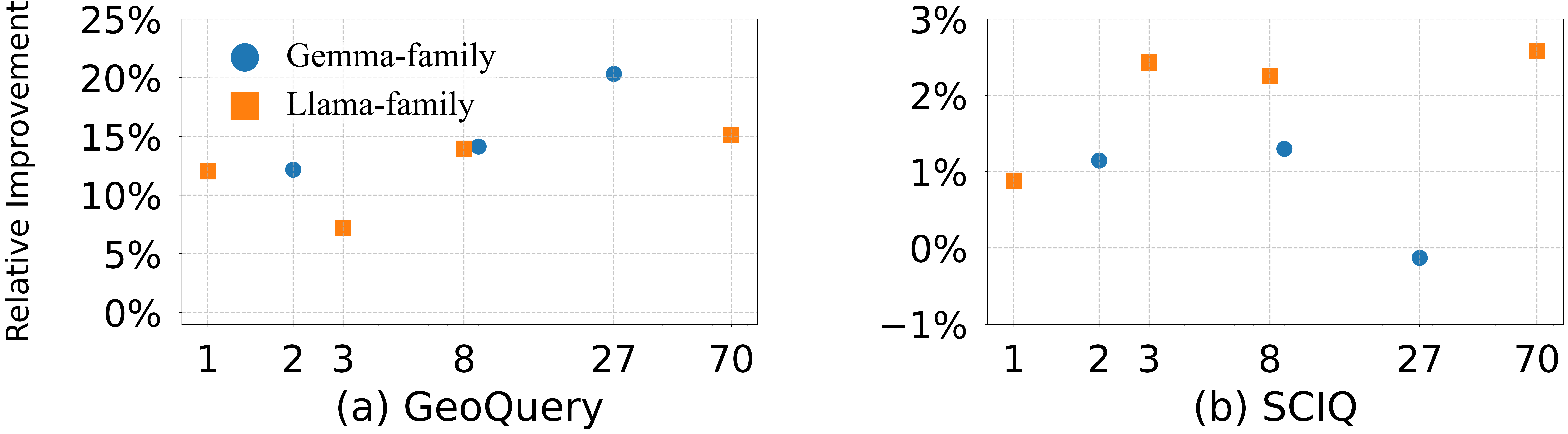}
    \caption{\small The relative improvement of diversity-aware methods over \topk.\textbf{ Left:} relative improvement of \topkdiv over \topk on GeoQuery standard split. \textbf{Right:} relative improvement of \diversity over \topk on SCIQ.}
    \label{fig:differ-scale}
\end{figure}

\Cref{fig:differ-scale} shows the relative improvement on GeoQuery standard split and SCIQ, where we observe the clear benefit of the diversity-aware method in \Cref{sec:main-findings}, on Llama-3.1/3.2 and Gemma-2 families with different sizes. We observe that in these two tasks, in general, the relative improvement does not decrease that much even if the model scales up, which indicates the importance of understanding the role of diversity in demonstration selection. 

In Appendix~\ref{sec:hyperparameter-ablation}, we present experiments on additional models.
We also report results across a range of settings, including fine-grained variations in $k$, different subset sizes for the \diversity method, fixed training sets, and changes in the embedding or decoding strategies.
In particular, we implement a \textbf{purely} diversity-based method, \kmeans, whose diversity score can exceed that of \diversity.
Its superior performance on Math and Reading tasks further supports Finding~1.
These ablation studies consistently reinforce our main findings, demonstrating the generality and robustness of our conclusions.

\section{Theoretical justification}\label{sec:theory}


In this section, we give a theoretical justification for combining diversity in demonstration selection for ICL, even if the ``embedding'' is accurate.


We consider the linear regression model, where there is a task vector $\theta_{\gT}\in\mathbb R^d$. The data for this task has embedded input $e\in\mathbb R^d$ and output $y = \langle \theta_{\gT}, e\rangle$. We also have a demonstration set $D = \{(e_i, y_i)\}_{i=1}^n$ with size $n$, where $y_i = \langle \theta_{\gT}, e_i\rangle$ and $e_i$ is drawn from the demonstration distribution $\gD_\gE$. Now given a query $e_q$ drawn from the query distribution $\gQ_\gE$, the goal of demonstration selection is to select a subset $S=\{(e_{j_i},y_{j_i})\}_{i=1}^K$, such that given the demonstrations $S=\{(e_{j_i},y_{j_i})\}_{i=1}^K$, the LLM predicts the output close to the gold label $y_q = \langle \theta_{\gT}, e_q\rangle$, i.e, $y_q \approx \text{LLM}(S, e_q)$. We make the following assumption on the mechanism of LLM for learning linear regression in-context, given the demonstration $S$ and the query $e_q$.

\begin{assumption}[\textbf{ICL for linear regression}]\label{ass:icl-lr}
    Suppose that the task $\gT$ is to predict the value of a linear function $y = \langle \theta_\gT, e\rangle$ and $K$ demonstrations $S=\{(e_{j_i},y_{j_i})\}_{i=1}^K$ are selected. Denote $E = [e_{j_1},\dots,e_{j_K}]^\top\in\mathbb R^{K\times d}$ as the data matrix. Then given a query $e_q$, we assume that the prediction given by the LLM is
    $\textstyle y_{\text{pred}} = \langle e_q, E^\dagger E\theta_{\gT}\rangle.$
    Namely, the LLM learns the min-norm solution for the overparameterized linear regression.
\end{assumption}


By this assumption, the prediction loss of $e_q$ is 
\[
\text{Loss}(e_q) := (y_{\text{pred}} - \langle \theta_\gT, e_q\rangle)^2 = \langle \theta_\gT - E^\dagger E\theta_{\gT}, e_q\rangle^2.
\]
ICL for linear regression has been extensively studied, empirically and theoretically (\Cref{sec:related-works}). \Cref{ass:icl-lr} is also empirically justified, where ~\citep{DBLP:conf/iclr/AkyurekSA0Z23} observed that after pretraining an autoregressive transformer model on noiseless linear regression tasks, the transformer will learn the min-norm solution for the linear regression in-context if the size of demonstrations $K < d$.

We further assume that the embedding for each data $e\in\{0,1\}^d$. This is inspired by the theoretical framework that each problem from a specific task contains certain skills (or local structures), and an LLM is able to solve that problem perfectly if the LLM knows all the skills (local structures) and is able to compose the skills (local structures) together~\citep{arora2023theory,DBLP:conf/iclr/Yu0GBGA24,zhao2025can}. For example, for a specific math problem related to algebra, the skills required to solve this problem are polynomial multiplication and solving equations, while for another math problem related to geometry, the skills required might be changed to coordinate systems and solving equations. It is also worth noting that the skill(local structure)-based embedding design also gains empirical success. For example, ~\citep{levy2023diverse,didolkar2024metacognitive} improves semantic parsing and ~\citep{an-etal-2023-skill} improves math ability by selecting demonstrations that require similar skills or local structures to the query.

\noindent \textbf{Example \rom{1}: Diversity benefits from coverage.}
We characterize the demonstration distribution $\gD_\gE$ and the query distribution $\gQ_\gE$ below.
Let $l \geq 200$ be an even number and let $d = 4l$, where the choice of 200 is to simplify the analysis.
Let $\gD_\gE$ be: Uniformly draw a subset $T_1\subseteq [2l]$ of size $l/2$ and a subset $T_2\subseteq \left\{2l+1,\ldots, 4l\right\}$ of size $l/2$, and output $e = e_{T_1\cup T_2}$, i.e., the $i$-th entry of $e$ is 1 iff $i\in T_1\cup T_2$.
Assume the size $n$ of $D$ is sufficiently large that $D$ covers the entire ground set of $\gD_\gE$.
Let $\gQ_\gE$ be: Uniformly draw a subset $T\subset [2l]$ of size $l$.
We have the following theorem, whose proof can be found in Appendix \ref{sec:proof_1}.

\begin{theorem}[\bf{Justification example \rom{1}}]
\label{thm:example_1}
Suppose each entry of $\theta_\gT$ is i.i.d. drawn from the uniform distribution on $[0,1]$.
Let $K = 2$ and $\gD_\gE, \gQ_\gE$ be as defined above.
For a query $e_q$ drawn from $\gQ_\gE$, let $L, L'$ denote the expected prediction loss of $e_q$ using \topk and \topkdiv, respectively, where the randomness comes from $\theta_\gT$, $e_q$, and the selection of demonstration examples.
Then $L> L'$ for any hyperparameter $\alpha\in (0,1)$ for \topkdiv.
\end{theorem}

Intuitively, the selected two demonstration examples of \topkdiv must cover all non-zero entries of $e_q$, while this property is unlikely to hold for \topk.
This demonstrates that adding diversity may increase the coverage of demonstration examples to queries and lead to a lower prediction loss, aligning with the findings in \cite{levy2023diverse,gupta2023coverage,ye2023complementary}.

\noindent \textbf{Example \rom{2}: Diversity is beyond coverage.}
We again characterize $\gD_\gE$ and $\gQ_\gE$ below.
Let $l \geq 3$ be an integer and let $d = 4l$.
Let $\gD_\gE$ be: Uniformly draw a subset $T_1\subseteq [2l]$ of size $l-1$ and a subset $T_2\subseteq \left\{2l+1,\ldots, 4l\right\}$ of size $1$, and output $e = e_{T_1\cup T_2}$.
Assume the size $n$ of $D$ is sufficiently large that $D$ covers the entire ground set of $\gD_\gE$.
Let $\gQ_\gE$ be: Uniformly draw a subset $T\subset [2l]$ of size $l$.
We have the following theorem for this example, whose proof can be found in Appendix \ref{sec:proof_2}.

\begin{theorem}[\bf{Justification example \rom{2}}]
\label{thm:example_2}
Suppose each entry of $\theta_\gT$ is i.i.d. drawn from the uniform distribution on $[0,1]$.
Let $K = 2$ and $\gD_\gE, \gQ_\gE$ be as defined above.
For a query $e_q$ drawn from $\gQ_\gE$, let $L, L'$ denote the expected prediction loss of $e_q$ using \topk and \topkdiv, respectively, where the randomness comes from $\theta_\gT$, $e_q$, and the selection of demonstration examples.
Then $L> L'$ if hyperparameter $\alpha\geq 1 - 1/l$ for \topkdiv.
\end{theorem}

The demonstration examples of \topk and \topkdiv must cover all non-zero entries of $e_q$.
The smaller loss of \topkdiv is caused by selecting two demonstration examples with different non-zero entries among $\left\{2l+1,\ldots, 4l\right\}$,
indicating that adding diversity could benefit ICL ``beyond coverage''. 

In Appendix \ref{sec:simulation}, we conduct simulations to validate that the advantage of \topkdiv\ over \topk, driven by coverage and beyond, extends to more general settings, including the ID setting (\(\gD_\gE = \gQ_\gE\)) and scenarios with different training scales for \( D \).

\section{Related works}\label{sec:related-works}

\paragraph{Demonstration selection}
Retrieval-based demonstration selection for ICL has long been studied, and the most notable methods are the \emph{similarity}-based methods~\citep{liu2021makes,DBLP:conf/aaai/YangGW0L0W22,wu2023self,qin2023context}. These are often augmented by trainable deep learning retrievers aimed at capturing core skills or features beyond mere semantic similarity~\citep{karpukhin2020dense,rubin2022learning,luo2023dr,scarlatos2023reticl,an-etal-2023-skill}, or by incorporating LLM feedback for refinement~\citep{li2023finding,chen2023relation,wang2023large}. Conversely, diversity-based, or more accurately,   coverage-based methods are less prevalent in retrieval-based selection. Existing studies in this vein typically address tasks with clear local structures where feature coverage is advantageous~\citep{levy2023diverse,ye2023complementary,gupta2023coverage,DBLP:conf/acl/AnLFC0L023}. 
More recently, \citep{zhan2025mmragmultimoderetrievalaugmentedgeneration} demonstrated the successful application of diversity-based methods in the biomedical domain, achieving superior performance over \topk strategies on more complex datasets. 
This also supports our Finding 1.

For non-retrieval-based ICL, where a fixed set of demonstrations is selected for a specific task, diversity is recognized as beneficial~\citep{DBLP:conf/iclr/0001Z0S23,gao2023constructing,su2023selective,yang2023representative}.

\paragraph{Understanding in-context learning}
Efforts to understand ICL span both theoretical and empirical investigations. Theoretical perspectives often frame ICL as either a Bayesian inference procedure~\citep{DBLP:conf/iclr/XieRL022,wanglarge,wies2023learnability,jiang2023latent,zhang2023and} or an implicit form of meta-optimization akin to gradient descent~\citep{dai2023can,von2023transformers,von2023uncovering,deutch2024context,shen2023pretrained}. Research on ICL for regression tasks~\citep{garg2022can,li2023transformers,li2023closeness,DBLP:conf/iclr/AkyurekSA0Z23} provides valuable insights; notably, \citep{DBLP:conf/iclr/AkyurekSA0Z23} suggest transformers can identify min-norm solutions in-context for linear regression, a finding that supports the role of demonstration diversity. Empirical studies have further examined factors such as input-label mapping~\citep{min2022rethinking,yoo2022ground,pan2023context}, the influence of demonstration order~\citep{lu2022fantastically,liu2024lost}, and the importance of calibration for ICL efficacy~\citep{zhao2021calibrate}.

\section{Conclusion, limitations, and future works}
\label{sec:conclusion}

We study the role of diversity in retrieval-based demonstration selection for in-context learning. By experimenting on different tasks, ranging from sentiment classification to math, on Llama-3, Gemma-2, and Mistral-v0.3 families, we observe the benefits of diversity when the task is challenging, the query is hard, and the query and demonstration might come from different distributions. These findings are also validated by ablation studies. 
We also provide theoretical justification for the benefits of diversity over similarity. 
Our findings and theoretical justification can help people better understand the role of diversity in demonstration selection for ICL.


However, the internal mechanism behind why diversity benefits, still remains unclear. Part of our findings can be explained by coverage, which is aligned with previous literature, but the superior performance on math, reading comprehension, and OOD generalization, cannot be explained by simply incentivizing coverage. 
Potential future research directions include both theoretical and empirical explorations into why diversity aids demonstration selection beyond coverage. This could involve deeper analysis of model representations, interactions between diverse demonstrations, or alternative explanations grounded in information theory or representation learning. 
Additionally, our diversity heuristic is tested on English text only; cross-lingual robustness is left for future work.

Our study advances the understanding of various data selection methods for in-context learning and may inform the design of more effective algorithms for future large language models, thereby contributing to positive societal impact.

\bibliographystyle{plainnat}
\bibliography{references}

\newpage
\appendix
\onecolumn

\tableofcontents

\newpage

\section{More experiment details}\label{sec:more-exp-details}

\subsection{Prompt template}

\Cref{tab:prompt-template} lists the template we use for different tasks. We take $K=2$ as an example.

\begin{table*}[htbp]
    \small
    \centering
    \caption{Prompt template for different tasks with 2 demonstrations. For Math problems, we also apply the chat template since we use the instruct models (done by applying the function ``apply\_chat\_template'' on the instruct models' tokenizer).}
    \label{tab:prompt-template}
    \begin{adjustbox}{width=0.9\textwidth,totalheight=8in}
    \begin{tabular}{m{80pt}|m{320pt}}
       Name & Template \\
       \hline
Sentiment Classification (SST-2, IMDB, Amazon) & \begin{lstlisting}[language=prompt,belowskip=-6pt]
Question: (@\intrprmpt{input\_1}@)
Answer: (@\intrprmpt{output\_1}@)

Question: (@\intrprmpt{input\_2}@)
Answer: (@\intrprmpt{output\_2}@)

Question: (@\intrprmpt{input\_query}@)
Answer:
\end{lstlisting} \\
\hline
Commonsense Reasoning (ARC-Easy, CsQA) & \begin{lstlisting}[language=prompt,belowskip=-6pt]
Question: (@\intrprmpt{input\_1}@)
Answer: (@\intrprmpt{output\_1}@)

Question: (@\intrprmpt{input\_2}@)
Answer: (@\intrprmpt{output\_2}@)

Question: (@\intrprmpt{input\_query}@)
Answer:
\end{lstlisting} \\
\hline
Reading Comprehension (SQuAD, SCIQ) & \begin{lstlisting}[language=prompt,belowskip=-6pt]
Support: (@\intrprmpt{support\_1}@)
Question: (@\intrprmpt{input\_1}@)
Answer: (@\intrprmpt{output\_1}@)

Support: (@\intrprmpt{support\_2}@)
Question: (@\intrprmpt{input\_2}@)
Answer: (@\intrprmpt{output\_2}@)

Support: (@\intrprmpt{support\_query}@)
Question: (@\intrprmpt{input\_query}@)
Answer:
\end{lstlisting} \\
\hline
text to SQL (GeoQuery) & \begin{lstlisting}[language=prompt,belowskip=-6pt]
Question: (@\intrprmpt{input\_1}@)
Answer: (@\intrprmpt{output\_1}@)

Question: (@\intrprmpt{input\_2}@)
Answer: (@\intrprmpt{output\_2}@)

Question: (@\intrprmpt{input\_query}@)
Answer:
\end{lstlisting} \\
\hline
Math (GSM8K, PRM800K) & \begin{lstlisting}[language=prompt,belowskip=-6pt]
Question: (@\intrprmpt{input\_1}@)
Answer: (@\intrprmpt{output\_1}@)

Question: (@\intrprmpt{input\_2}@)
Answer: (@\intrprmpt{output\_2}@)

Let's think step by step. You need to solve the final question and answer in the format: \n#### \{result\}
Question: (@\intrprmpt{input\_query}@)
Answer:
\end{lstlisting} \\
    \end{tabular}
    \end{adjustbox}
\end{table*}

\subsection{Dataset details}
\begin{table*}[!t]
\centering
\small
\caption{\textbf{Detailed dataset size before and after sampling.} We show the original and sampled size of demonstration set and test set for all dataset we considered.}
\label{tab:dataset-size}
\begin{adjustbox}{width=\textwidth}
\begin{tabular}{ccccccccccccc}
\toprule
\multirow{2}{*}{Dataset size} & \multicolumn{3}{c}{Classification} & \multicolumn{2}{c}{Multi-choice} & \multicolumn{3}{c}{Math} & Code & \multicolumn{2}{c}{Reading} & \\
& SST-2 & Amazon & Imdb& ARC-Easy & CsQA & PRM800K & GSM8K & GSM-Plus-Mini & GeoQuery & SQuAD & SCIQ & \\
\midrule
Sampled demo set & 1000& 1000 & 1000 & 1000 & 1000 & 12000 & 7473 & 7473 & 600 & 10000 & 1000 \\
Sampled test set & 1000& 1000 & 1000 & 1000 & 1000 & 500 & 1319 & 2400 & 280 & 1000 & 1000 \\
\midrule
Original demo set & 67300& 3600000 & 25000 & 2250 &  9740 & 12000 & 7473 & 7473 & 600 & 87600 & 11700 \\
Original test set & 1820& 400000 & 25000 &  2380 & 1140 & 500 & 1319 & 2400 & 280 & 10600 & 1000 \\
\bottomrule
\end{tabular}
\end{adjustbox}
\end{table*}

To reduce computational cost, we performed random sampling on both the \emph{demo} and \emph{test} set for classification, multi-choice and reading tasks.
For classification tasks, the sampled datasets from IMDB and SST2 are consistent with ~\cite{chang2023data}.
A fixed random seed of 42 was used for all sampling procedures.
For math tasks, since the \emph{test} set sizes of PRM800K and GSM8K datasets are close to the sampled \emph{test} set sizes of other tasks, we directly used their existing \emph{demo} and \emph{test} set.
Detailed sampling statistics are provided in \Cref{tab:dataset-size}.

\subsection{Evaluation details}

For the sentiment classification task (classification), given the prompt listed in \Cref{tab:prompt-template}, we compute the logit for ``great'' and ``terrible'' respectively, and predict the sentiment to be positive is the logit for ``great'' is larger than that for ``terrible'', and vice versa. We report the accuracy metric.

For commonsense reasoning tasks (multi-choice), given the prompt, we compute the average cross-entropy loss on each given option, conditioned on the prompt. Then we pick the option with the smallest average cross-entropy loss. We report the accuracy metric.

For reading comprehension (generation), given the prompt, we generate the answer using greedy decoding. We stop if we generate one of the following string: ``\textbackslash n\textbackslash n'', ``\textbackslash n\textbackslash n\textbackslash n'', "Support", "Support:", "Question", "Question:". We compare the generated answer with the gold answer, and report the exact match metric. There are several optional answers for the squad test sample, if the generated answer exactly matches one of them, we consider it correct.

For text to SQL (generation), given the prompt, we generate the answer using greedy decoding. We stop if we generate one of the following string: ``\textbackslash n\textbackslash n'', ``\textbackslash n\textbackslash n\textbackslash n'', "Question", "Question:". We compare the generated answer with the gold answer, and report the exact match metric.

For math problem (generation), given the prompt, we generate the answer using greedy decoding. We do not stop the generation process unless the instruct model generates the stop sign itself. We first try to extract the math expression from the following format ``\#\#\#\# \{expression\}''. If failed, we try to extract from the following format ``\textbackslash\{boxed\}\{expression\}''. If both failed, we extract the final math expression from the answer. The report exact match metric.

For each task, the selected examples in TopK and TopK-Div are fixed, and these two methods are tested once.
For Rand and Div, where example selection involves randomness, we test with ten random seeds and report the average results.



%
%
%

\section{Additional ablation studies}\label{sec:hyperparameter-ablation}

\subsection{Results on more models}\label{sec:res-full}
\begin{table*}[!t]
\centering
\small
\caption{We supplemented the  content omitted in \Cref{tab:main-res}. The main numerical values represent the mean results over ten random seeds, while the subscript indicates their $\mathrm{std}$. We still highlight the result with the highest mean in bold. In most cases, the fluctuations within each method do not affect our conclusions.
}
\label{tab:detailed-main-res}
\begin{adjustbox}{width=0.95\textwidth}
\begin{tabular}{cccccccccccc}
\toprule        
\multirow{2}{*}{Model} & \multicolumn{1}{c}{\multirow{2}{*}{$K$}} & \multirow{2}{*}{Method}  & \multicolumn{2}{c}{Classification} & \multicolumn{2}{c}{Multi-choice}  & \multicolumn{2}{c}{Math} & Code & \multicolumn{2}{c}{Reading}\\
& & & SST-2 & Amazon & ARC-Easy & CsQA & GSM8K & GSM-Plus-Mini & GeoQuery & SQuAD & SCIQ \\
\midrule
\parbox[t]{3mm}{\multirow{8}{*}{\rotatebox[origin=c]{90}{Llama-3.1-8B}}} 
& \multirow{1}{*}{0} 
& - & $87.50$ & $95.40$ & $82.43$ & $62.80$ & $53.45$ & $65.12$ & --- & $42.30$& $36.40$\\
\cline{2-12}
& \multirow{4}{*}{4} 
& \rand & $91.31_{0.59}$ & $\mathbf{96.38_{0.25}}$ & $84.72_{0.35}$& $71.15_{0.65}$& $\mathbf{82.24_{0.52}}$& $66.90_{0.59}$ &  $12.57_{1.33}$& $\mathbf{75.95_{0.55}}$& $74.00_{0.57}$\\
& & \topk &  $\mathbf{94.13_{0.21}}$&$96.24_{0.16}$&   $\mathbf{86.10_{0.32}}$& $72.54_{0.36}$& $81.99_{0.55}$ & $65.30_{0.46}$ & $63.04_{1.96}$& $73.51_{0.48}$& $72.70_{0.40}$\\
& & \diversity & $91.50_{0.63}$&$96.18_{0.25}$ & $85.06_{0.27}$&  $71.17_{0.42}$&  $82.14_{0.45}$&  $\mathbf{66.92_{0.52}}$ &  $33.71_{1.35}$& $75.66_{0.97}$& $\mathbf{74.47_{0.62}}$\\
& & \topkdiv& $92.75_{0.33}$ & $96.15_{0.22}$&  $85.83_{0.38}$& $\mathbf{72.57_{0.35}}$& $81.74_{0.53}$& $66.12_{0.85}$ &  $\mathbf{71.07_{1.11}}$&$73.28_{0.79}$ & $73.87_{0.35}$\\
\cline{2-12}
& \multirow{4}{*}{8} 
& \rand & $92.27_{0.55}$ &$\mathbf{96.63_{0.27}}$ & $84.38_{0.34}$& $72.23_{0.34}$& $82.81_{0.61}$ & $\mathbf{66.72_{0.72}}$ & $23.21_{1.41}$& $77.13_{0.80}$& $74.65_{0.88}$\\
& & \topk & $\mathbf{93.64_{0.36}}$&$96.12_{0.09}$ & $\mathbf{85.91_{0.29}}$& $\mathbf{73.91_{0.38}}$& $82.26_{0.65}$ & $65.99_{0.60}$ & $72.04_{0.93}$&  $75.52_{0.43}$& $74.72_{0.65}$\\
& & \diversity & $92.95_{0.35}$ & $96.25_{0.19}$ & $84.97_{0.32}$& $72.77_{0.61}$& $\mathbf{82.98_{0.34}}$ & $66.56_{0.60}$ & $38.54_{0.90}$& $\mathbf{77.71_{0.80}}$& $\mathbf{75.17_{0.53}}$ \\
& & \topkdiv & $93.33_{0.36}$ & $96.43_{0.09}$ & $85.39_{0.40}$ & $73.76_{0.37}$ & $82.63_{0.57}$ & $66.48_{0.52}$ & $\mathbf{78.36_{1.24}}$ & $76.13_{0.42}$ & $75.07_{0.49}$\\
\midrule
\parbox[t]{3mm}{\multirow{8}{*}{\rotatebox[origin=c]{90}{Gemma-2-9B}}} 
& \multirow{1}{*}{0} 
& - & $67.50$ & $85.10$ & $88.15$ & $61.80$ & $16.07$& $32.79$ & ---& $37.90$& $41.10$\\
\cline{2-12}
& \multirow{4}{*}{4} 
& \rand & $93.33_{0.52}$& $96.15_{0.23}$& $89.52_{0.25}$& $74.70_{0.70}$& $84.29_{0.43}$ & $74.40_{0.47}$ & $13.89_{1.67}$& $\mathbf{77.19_{0.89}}$& $75.80_{0.54}$\\
& & \topk & $\mathbf{94.47_{0.48}}$& $96.34_{0.20}$& $\mathbf{90.50_{0.16}}$& $75.19_{0.25}$& $84.25_{0.73}$& $\mathbf{74.50_{0.55}}$ & $61.14_{1.33}$& $74.82_{0.70}$& $75.24_{0.34}$\\
& & \diversity & $93.45_{0.46}$& $95.69_{0.23}$& $90.03_{0.24}$& $74.85_{0.39}$& $\mathbf{84.44_{0.91}}$& $73.34_{0.62}$ &$36.29_{1.05}$ & $77.06_{0.57}$& $\mathbf{75.96_{0.55}}$\\
& & \topkdiv& $93.34_{0.34}$& $\mathbf{96.57_{0.16}}$& $90.19_{0.19}$& $\mathbf{75.60_{0.54}}$& $83.54_{0.56}$& $74.47_{0.63}$ & $\mathbf{70.43_{1.24}}$& $75.05_{0.41}$& $75.21_{0.29}$\\
\cline{2-12}
& \multirow{4}{*}{8} 
& \rand & $93.30_{0.36}$& $96.09_{0.23}$& $89.39_{0.28}$& $75.98_{0.56}$& $\mathbf{84.34_{0.54}}$& $74.48_{0.63}$ & $24.36_{1.19}$& $\mathbf{79.23_{0.64}}$& $76.28_{0.50}$\\
& & \topk & $\mathbf{94.20_{0.28}}$& $96.55_{0.16}$& $\mathbf{90.62_{0.16}}$& $76.14_{0.63}$& $83.57_{0.53}$& $\mathbf{75.36_{0.43}}$ & $71.00_{1.20}$& $77.59_{0.42}$& $75.55_{0.18}$\\
& & \diversity & $93.41_{0.20}$& $95.94_{0.25}$& $89.90_{0.19}$& $\mathbf{76.60_{0.32}}$& $84.22_{0.52}$& $74.69_{0.64}$ & $42.07_{1.10}$& $79.05_{0.93}$& $\mathbf{76.65_{0.60}}$\\
& & \topkdiv & $94.04_{0.29}$& $\mathbf{96.58_{0.04}}$& $90.48_{0.22}$& $76.53_{0.21}$ & $83.85_{0.66}$& $75.16_{0.32}$ &$\mathbf{76.32_{0.85}}$ & $77.64_{0.63}$& $76.24_{0.48}$\\
\midrule
\parbox[t]{3mm}{\multirow{8}{*}{\rotatebox[origin=c]{90}{Mistral-7B-v0.3}}} 
& \multirow{1}{*}{0} 
& - & $66.50$ & $94.00$ & $76.41$ & $51.80$& $9.48$ & $5.17$ & --- &$30.50$ &$34.20$ \\
\cline{2-12}
&  \multirow{4}{*}{4} 
& \rand & $91.00_{0.78}$& $94.02_{0.61}$& $82.77_{0.48}$ & $69.83_{0.81}$ & $48.78_{1.00}$ &  $37.20_{0.69}$ & $12.14_{1.47}$&  $\mathbf{76.70_{0.72}}$ & $74.71_{0.54}$\\
& & \topk & $\mathbf{93.57_{0.25}}$& $\mathbf{96.17_{0.20}}$& $\mathbf{85.21_{0.30}}$ & $69.73_{0.43}$ & $49.28_{1.17}$ & $38.20_{0.55}$ & $60.14_{0.82}$&  $75.04_{0.74}$ & $73.73_{0.59}$\\
& & \diversity & $91.98_{0.46}$& $94.15_{0.31}$& $82.98_{0.25}$ & $\mathbf{70.15_{0.56}}$ & $49.49_{0.87}$ & $37.50_{0.76}$ &  $34.89_{1.39}$& $75.96_{1.08}$ & $\mathbf{75.83_{0.57}}$ \\
& & \topkdiv& $92.73_{0.30}$& $95.90_{0.15}$& $84.55_{0.20}$ & $69.91_{0.49}$ & $\mathbf{49.99_{1.02}}$ & $\mathbf{38.45_{0.81}}$ &$\mathbf{71.46_{1.35}}$ &   $74.43_{0.50}$& $73.16_{0.28}$\\
\cline{2-12}
& \multirow{4}{*}{8} 
& \rand & $92.49_{0.34}$& $95.35_{0.36}$& $83.69_{0.36}$ & $71.65_{0.60}$ &$47.86_{1.19}$ & $36.32_{0.71}$ & $22.18_{1.96}$& $77.30_{0.54}$& $75.54_{0.63}$\\
& & \topk & $\mathbf{93.61_{0.32}}$& $\mathbf{96.15_{0.16}}$& $\mathbf{85.17_{0.25}}$ &  $71.88_{0.38}$& $48.43_{1.02}$ & $37.35_{0.53}$ & $70.50_{1.36}$& $77.05_{0.41}$ &  $75.44_{0.43}$\\
& & \diversity & $92.55_{0.29}$ & $95.10_{0.37}$ & $84.27_{0.41}$ & $\mathbf{72.04_{0.61}}$ & $48.33_{1.10}$ & $36.12_{0.34}$ & $39.14_{1.40}$ & $\mathbf{77.67_{1.56}}$  & $\mathbf{76.30_{0.31}}$\\
& & \topkdiv & $93.47_{0.41}$ & $96.11_{0.16}$ & $84.85_{0.34}$ & $71.81_{0.19}$ & $\mathbf{48.60_{0.71}}$ & $\mathbf{37.81_{0.76}}$ & $\mathbf{77.93_{1.70}}$ &   $77.44_{0.37}$ & $75.22_{0.42}$\\
\bottomrule
\end{tabular}
\end{adjustbox}
\end{table*}

\begin{table*}[htbp]
    \centering
    \small
    \caption{Performance of different algorithms for models belong to Llama-family. Setting same as \Cref{tab:main-res} while adding results from more models (Llama-3.2-1B, Llama-3.2-3B, Llama-3.1-70B). Our finding that diversity helps for more challenging tasks still holds.}
    \label{tab:Llama-family}
    \begin{adjustbox}{width=\textwidth}
    \begin{tabular}{cccccccccccc}
    \toprule
    \multirow{2}{*}{model} & \multirow{2}{*}{$K$} & \multirow{2}{*}{Method}  & \multicolumn{2}{c}{Classification} & \multicolumn{2}{c}{Multi-choice}  & \multicolumn{2}{c}{Math} & Code & \multicolumn{2}{c}{Reading}\\
    & & & SST-2 & Amazon & Arc-easy & CsQA & GSM8K & GSM-Plus-Mini & GeoQuery & SQuAD & SCIQ \\
    \midrule
    \parbox[t]{3mm}{\multirow{8}{*}{\rotatebox[origin=c]{90}{Llama-3.2-1B}}} & \multirow{4}{*}{4} 
    & \rand & $86.88_{0.49}$ & $90.64_{0.43}$ & $71.65_{0.46}$ & $59.54_{0.63}$ & $\mathbf{22.18_{0.81}}$ & \textbf{$12.62_{0.56}$}& $7.43_{1.27}$& $56.17_{0.75}$& $62.75_{0.84}$\\
    & & \topk & $\mathbf{91.87_{0.49}}$ & $\mathbf{93.46_{0.25}}$ & $\mathbf{75.28_{0.57}}$ & $60.20_{0.46}$ & $19.57_{0.71}$ & $9.41_{0.60}$& $48.25_{0.94}$& $55.09_{0.38}$& $\mathbf{63.65_{0.36}}$\\
    & & \diversity & $88.35_{1.19}$ & $91.14_{0.43}$ & $73.52_{0.54}$ & $60.60_{0.65}$ & $21.29_{1.05}$ & $\mathbf{13.16_{0.73}}$& $29.46_{1.70}$& $55.40_{1.18}$& $63.38_{0.54}$\\
    & & \topkdiv & $91.47_{0.60}$ & $93.22_{0.33}$ & $74.62_{0.44}$ & $\mathbf{60.89_{0.36}}$ & $20.30_{0.80}$ & $9.35_{0.53}$& $\mathbf{56.57_{1.18}}$& $\mathbf{56.39_{0.96}}$& $62.96_{0.31}$\\
    \cline{2-12}
    & \multirow{4}{*}{8} 
    & \rand & $89.56_{0.81}$ & $92.62_{0.39}$ & $72.72_{0.32}$ & $61.27_{0.62}$ & $21.04_{0.63}$ & $10.01_{0.44}$& $13.04_{1.71}$& $\mathbf{58.76_{0.56}}$& $65.05_{0.45}$\\
    & & \topk & $\mathbf{92.91_{0.30}}$ & $93.95_{0.24}$ & $\mathbf{75.41_{0.40}}$ & $61.77_{0.35}$ & $16.58_{0.53}$ & $7.12_{0.44}$& $56.29_{1.94}$& $58.38_{0.66}$& $65.87_{0.45}$\\
    & & \diversity & $87.96_{1.14}$ & $92.72_{0.34}$ & $73.53_{0.45}$ & $\mathbf{62.24_{0.55}}$ & $\mathbf{22.24_{0.93}}$ & $\mathbf{11.73_{1.55}}$& $32.75_{1.46}$& $58.37_{1.16}$& $65.89_{0.94}$\\
    & & \topkdiv & $92.30_{0.38}$ & $\mathbf{94.06_{0.20}}$ & $74.74_{0.28}$ & $62.20_{0.50}$ & $16.00_{0.81}$ & $6.45_{0.46}$& $\mathbf{65.11_{1.63}}$ & $58.27_{0.74}$ & $\mathbf{66.24_{0.57}}$\\
    \midrule
    \parbox[t]{3mm}{\multirow{8}{*}{\rotatebox[origin=c]{90}{Llama-3.2-3B}}} & \multirow{4}{*}{4} 
    & \rand & $90.40_{0.66}$ & $95.87_{0.21}$ & $78.62_{0.44}$ & $68.51_{0.64}$ & $69.64_{0.98}$ & $50.50_{0.59}$& $9.86_{1.28}$& $\mathbf{71.59_{0.60}}$ & $72.43_{0.70}$\\
    & & \topk & $92.87_{0.31}$ & $96.25_{0.24}$ & $81.51_{0.34}$ & $68.72_{0.37}$ & $\mathbf{70.05_{0.86}}$ & $50.10_{0.53}$& $54.04_{1.68}$& $71.21_{0.48}$& $70.87_{0.46}$\\
    & & \diversity & $90.87_{0.59}$ & $95.57_{0.18}$ & $80.41_{0.60}$ & $68.80_{0.57}$ & $68.71_{0.75}$ & $\mathbf{51.53_{0.87}}$& $31.21_{1.82}$& $71.13_{1.42}$& $\mathbf{72.58_{0.61}}$\\
    & & \topkdiv & $\mathbf{93.03_{0.29}}$ & $\mathbf{96.43_{0.15}}$ & $\mathbf{81.71_{0.30}}$ & $\mathbf{68.95_{0.25}}$ & $69.14_{0.86}$ & $50.60_{0.38}$& $\mathbf{59.75_{2.03}}$& $70.73_{0.49}$& $71.28_{0.40}$\\
    \cline{2-12}
    & \multirow{4}{*}{8} 
    & \rand & $91.79_{0.36}$ & $96.09_{0.14}$ & $78.91_{0.40}$ & $69.89_{0.68}$ & $68.79_{0.94}$ & $ \mathbf{51.12_{0.81}}$ & $19.29_{1.60}$& $73.14_{0.63}$& $72.57_{0.85}$ \\
    & & \topk & $\mathbf{93.71_{0.40}}$ & $96.18_{0.20}$ & $81.03_{0.36}$ & $\mathbf{70.53_{0.33}}$ & $\mathbf{69.40_{0.89}}$ & $50.00_{0.84}$ & $61.89_{1.49}$& $72.61_{0.30}$& $71.83_{0.46}$\\
    & & \diversity & $91.99_{0.47}$ & $95.83_{0.27}$ & $80.68_{0.34}$ & $70.26_{0.40}$ & $66.54_{1.14}$ & $50.87_{0.61}$& $36.71_{1.39}$& $72.76_{1.53}$& $\mathbf{74.07_{0.49}}$\\
    & & \topkdiv & $93.55_{0.35}$ & $\mathbf{96.37_{0.17}}$ & $\mathbf{81.57_{0.30}}$ & $70.18_{0.33}$ & $69.38_{0.98}$ & $49.96_{0.73}$& $\mathbf{72.11_{1.77}}$& $\mathbf{73.80_{0.56}}$ & $72.08_{0.53}$\\
    \midrule
    \parbox[t]{3mm}{\multirow{8}{*}{\rotatebox[origin=c]{90}{Llama-3.1-8B}}} & \multirow{4}{*}{4} 
    & \rand & $91.31_{0.59}$ & $\mathbf{96.38_{0.25}}$ & $84.72_{0.35}$ & $71.15_{0.65}$ & $\mathbf{82.24_{0.52}}$ & $66.90_{0.59}$& $12.57_{1.33}$& $\mathbf{75.95_{0.55}}$& $74.00_{0.57}$\\
    & & \topk & $\mathbf{94.13_{0.21}}$ & $96.24_{0.16}$ & $\mathbf{86.10_{0.32}}$ & $72.54_{0.36}$ & $81.99_{0.55}$ & $65.30_{0.46}$& $63.04_{1.96}$& $73.51_{0.48}$& $72.70_{0.40}$\\
    & & \diversity & $91.50_{0.63}$ & $96.18_{0.25}$ & $85.06_{0.27}$ & $71.17_{0.42}$ & $82.14_{0.45}$ & $\mathbf{66.92_{0.52}}$& $33.71_{1.35}$& $75.66_{0.97}$& $\mathbf{74.47_{0.62}}$\\
    & & \topkdiv & $92.75_{0.33}$ & $96.15_{0.22}$ & $85.83_{0.38}$ & $\mathbf{72.57_{0.35}}$ & $81.74_{0.53}$ & $66.12_{0.85}$& $\mathbf{71.07_{1.11}}$& $73.28_{0.79}$& $73.87_{0.35}$\\
    \cline{2-12}
    & \multirow{4}{*}{8} 
    & \rand & $92.27_{0.55}$ & $\mathbf{96.63_{0.27}}$ & $84.38_{0.34}$ & $72.23_{0.34}$ & $82.81_{0.61}$ & $\mathbf{66.72_{0.72}}$& $23.21_{1.41}$ & $77.13_{0.80}$& $74.65_{0.88}$\\
    & & \topk & $\mathbf{93.64_{0.36}}$ & $96.12_{0.09}$ & $\mathbf{85.91_{0.29}}$ & $\mathbf{73.91_{0.38}}$ & $82.26_{0.65}$ & $65.99_{0.60}$& $72.04_{0.93}$ & $75.52_{0.43}$& $74.72_{0.65}$\\
    & & \diversity & $92.95_{0.35}$ & $96.25_{0.19}$ & $84.97_{0.32}$ & $72.77_{0.61}$ & $\mathbf{82.98_{0.34}}$ & $66.56_{0.60}$& $38.54_{0.90}$ & $\mathbf{77.71_{0.80}}$& $\mathbf{75.17_{0.53}}$\\
    & & \topkdiv & $93.33_{0.36}$ & $96.43_{0.09}$ & $85.39_{0.40}$ & $73.76_{0.37}$ & $82.63_{0.57}$ & $66.48_{0.52}$& $\mathbf{78.36_{1.24}}$& $76.13_{0.42}$& $75.07_{0.49}$\\
    \midrule
    \parbox[t]{3mm}{\multirow{8}{*}{\rotatebox[origin=c]{90}{Llama-3.1-70B}}} & \multirow{4}{*}{4} 
    & \rand & $94.16_{0.33}$ & $96.77_{0.38}$ & $89.76_{0.16}$ & $75.48_{0.62}$ & $88.64_{0.48}$ & $77.14_{0.39}$& $17.50_{1.88}$& $\mathbf{81.47_{0.75}}$ & $75.51_{0.87}$\\
    & & \topk & $\mathbf{94.81_{0.34}}$ & $96.86_{0.11}$ & $\mathbf{90.57_{0.28}}$ & $\mathbf{76.22_{0.28}}$ & $88.87_{0.52}$ & $76.19_{0.57}$& $66.46_{1.23}$& $79.15_{0.22}$ & $75.67_{0.38}$\\
    & & \diversity & $94.34_{0.28}$ & $96.36_{0.23}$ & $90.14_{0.30}$ & $75.53_{0.33}$ & $\mathbf{89.27_{0.53}}$ & $\mathbf{77.21_{0.47}}$& $39.00_{1.87}$& $81.27_{1.25}$ & $\mathbf{77.75_{0.44}}$\\
    & & \topkdiv & $94.20_{0.18}$ & $\mathbf{96.88_{0.12}}$ & $90.46_{0.32}$ & $76.21_{0.49}$ & $88.67_{0.42}$ & $76.94_{0.45}$& $\mathbf{77.32_{0.95}}$& $79.26_{0.37}$ & $75.59_{0.33}$\\
    \cline{2-12}
    & \multirow{4}{*}{8} 
    & \rand & $94.66_{0.36}$ & $96.95_{0.28}$ & $89.84_{0.28}$ & $77.14_{0.40}$ & $89.47_{0.54}$ & $76.93_{0.77}$& $26.89_{1.90}$ & $82.62_{0.47}$ & $76.56_{0.78}$\\
    & & \topk & $94.18_{0.32}$ & $96.95_{0.18}$ & $90.24_{0.25}$ & $\mathbf{77.65_{0.30}}$ & $89.33_{0.29}$ & $76.29_{0.46}$& $75.68_{1.08}$ & $81.38_{0.51}$ & $76.70_{0.63}$\\
    & & \diversity & $\mathbf{94.95_{0.30}}$ & $96.47_{0.25}$ & $89.99_{0.17}$ & $77.33_{0.59}$ & $\mathbf{89.65_{0.27}}$ & $\mathbf{77.11_{0.29}}$& $44.25_{1.92}$ & $\mathbf{83.18_{1.43}}$ & $\mathbf{78.37_{0.60}}$\\
    & & \topkdiv & $94.75_{0.19}$ & $\mathbf{97.27_{0.11}}$ & $\mathbf{90.71_{0.25}}$ & $77.24_{0.31}$ & $89.17_{0.42}$ & $76.74_{0.92}$& $\mathbf{81.39_{1.16}}$& $81.47_{0.23}$ & $76.77_{0.41}$\\
    \bottomrule
    \end{tabular}
    \end{adjustbox}
    \end{table*}

\begin{table*}[htbp]
    \centering
    \small
    \caption{Performance of different algorithms for models belong to Gemma-family. Setting same as \Cref{tab:main-res} while adding results from more models (Gemma-2-2b and Gemma-2-27b). Our finding that diversity helps for more challenging tasks still holds.}
    \label{tab:Gemma-family}
    \begin{adjustbox}{width=\textwidth}
    \begin{tabular}{cccccccccccc}
    \toprule
    \multirow{2}{*}{model} & \multirow{2}{*}{$K$} & \multirow{2}{*}{Method}  & \multicolumn{2}{c}{Classification} & \multicolumn{2}{c}{Multi-choice}  & \multicolumn{2}{c}{Math} & Code & \multicolumn{2}{c}{Reading}\\
    & & & SST-2 & Amazon & Arc-easy & CsQA & GSM8K & GSM-Plus-Mini & GeoQuery & SQuAD & SCIQ \\
    \midrule
    \parbox[t]{3mm}{\multirow{8}{*}{\rotatebox[origin=c]{90}{Gemma-2-2B}}} & \multirow{4}{*}{4} 
    & \rand & $85.00_{0.94}$& $92.34_{0.75}$& $82.84_{0.45}$& $68.89_{0.61}$& $40.45_{1.11}$& $33.43_{0.69}$& $9.14_{1.38}$& $\mathbf{69.03_{0.34}}$& $71.27_{0.63}$\\
    & & \topk & $90.67_{0.55}$& $\mathbf{95.14_{0.18}}$& $\mathbf{84.57_{0.35}}$& $69.93_{0.47}$& $41.33_{0.85}$& $\mathbf{34.39_{0.71}}$& $53.39_{1.88}$& $68.19_{0.76}$& $71.09_{0.49}$\\
    & & \diversity & $89.63_{0.54}$& $92.55_{0.44}$& $84.21_{0.37}$& $\mathbf{71.03_{0.62}}$& $40.53_{2.15}$& $32.75_{1.64}$& $31.29_{1.47}$& $67.73_{1.18}$& $\mathbf{72.34_{0.45}}$\\
    & & \topkdiv & $\mathbf{91.66_{0.57}}$& $95.01_{0.22}$& $84.51_{0.32}$& $70.72_{0.52}$& $\mathbf{42.74_{0.57}}$& $\mathbf{34.39_{0.64}}$& $\mathbf{61.04_{1.55}}$& $67.85_{0.60}$& $71.64_{0.30}$\\
    \cline{2-12}
    & \multirow{4}{*}{8} 
    & \rand & $89.96_{0.51}$& $94.10_{0.47}$& $82.62_{0.33}$& $70.26_{0.34}$& $36.44_{0.82}$& $34.55_{0.46}$& $16.68_{2.01}$& $69.99_{0.68}$& $72.12_{0.29}$\\
    & & \topk & $92.22_{0.46}$& $\mathbf{95.58_{0.23}}$& $84.30_{0.19}$& $71.06_{0.44}$& $\mathbf{43.05_{0.69}}$& $36.45_{0.42}$& $61.00_{1.44}$& $69.15_{0.44}$& $\mathbf{72.35_{0.44}}$\\
    & & \diversity & $91.81_{0.48}$& $94.57_{0.37}$& $84.37_{0.38}$& $\mathbf{72.22_{0.62}}$& $38.87_{1.28}$& $34.07_{1.12}$& $35.29_{1.25}$& $69.31_{0.99}$& $72.28_{0.43}$\\
    & & \topkdiv & $\mathbf{92.40_{0.28}}$& $95.53_{0.20}$& $\mathbf{84.39_{0.24}}$& $71.99_{0.36}$& $42.93_{0.66}$& $\mathbf{36.47_{0.47}}$& $\mathbf{68.93_{1.35}}$& $\mathbf{70.09_{0.54}}$& $72.30_{0.37}$\\
    \midrule
    \parbox[t]{3mm}{\multirow{8}{*}{\rotatebox[origin=c]{90}{Gemma-2-9B}}} & \multirow{4}{*}{4} 
    & \rand      & $93.33_{0.52}$& $96.15_{0.23}$& $89.52_{0.25}$& $74.70_{0.70}$& $84.29_{0.43}$& $74.40_{0.47}$& $13.89_{1.67}$& $\mathbf{77.19_{0.89}}$ & $75.80_{0.54}$ \\
    & & \topk    & $\mathbf{94.47_{0.48}}$& $96.34_{0.20}$& $\mathbf{90.50_{0.16}}$& $75.19_{0.25}$& $84.25_{0.73}$& $\mathbf{74.50_{0.55}}$& $61.14_{1.33}$& $74.82_{0.70}$ & $75.24_{0.34}$ \\
    & & \diversity & $93.45_{0.46}$& $95.69_{0.23}$& $90.03_{0.24}$& $74.85_{0.39}$& $\mathbf{84.44_{0.91}}$& $73.34_{0.62}$& $36.29_{1.05}$& $77.06_{0.57}$ & $\mathbf{75.96_{0.55}}$ \\
    & & \topkdiv & $93.34_{0.34}$& $\mathbf{96.57_{0.16}}$& $90.19_{0.19}$& $\mathbf{75.60_{0.54}}$& $83.54_{0.56}$& $74.47_{0.63}$& $\mathbf{70.43_{1.24}}$& $75.05_{0.41}$ & $75.21_{0.29}$ \\
    \cline{2-12}
    & \multirow{4}{*}{8} 
    & \rand      & $93.30_{0.36}$& $96.09_{0.23}$& $89.39_{0.28}$& $75.98_{0.56}$& $\mathbf{84.34_{0.54}}$& $74.48_{0.63}$& $24.36_{1.19}$& $\mathbf{79.23_{0.64}}$ & $76.28_{0.50}$ \\
    & & \topk    & $\mathbf{94.20_{0.28}}$& $96.55_{0.16}$& $\mathbf{90.62_{0.16}}$& $76.14_{0.63}$& $83.57_{0.53}$& $\mathbf{75.36_{0.43}}$& $71.00_{1.20}$& $77.59_{0.42}$ & $75.55_{0.18}$ \\
    & & \diversity     & $93.41_{0.20}$& $95.94_{0.25}$& $89.90_{0.19}$& $\mathbf{76.60_{0.32}}$& $84.22_{0.52}$& $74.69_{0.64}$& $42.07_{1.10}$& $79.05_{0.93}$ & $\mathbf{76.65_{0.60}}$ \\
    & & \topkdiv & $94.04_{0.29}$& $\mathbf{96.58_{0.04}}$& $90.48_{0.22}$& $76.53_{0.21}$& $83.85_{0.66}$& $75.16_{0.32}$& $\mathbf{76.32_{0.85}}$& $77.64_{0.63}$ & $76.24_{0.48}$ \\
    \midrule
    \parbox[t]{3mm}{\multirow{8}{*}{\rotatebox[origin=c]{90}{Gemma-2-27B}}} & \multirow{4}{*}{4} 
    & \rand & $94.16_{0.40}$& $96.06_{0.26}$& $\mathbf{89.99_{0.39}}$& $76.15_{0.41}$& $90.16_{0.33}$& $\mathbf{70.76_{0.71}}$& $18.68_{1.83}$& $\mathbf{80.54_{0.59}}$& $75.61_{0.78}$\\
    & & \topk & $\mathbf{95.00_{0.33}}$& $96.47_{0.11}$& $89.64_{0.15}$& $76.47_{0.46}$& $89.73_{0.38}$& $69.27_{0.49}$& $67.75_{1.45}$& $78.43_{0.49}$& $76.35_{0.46}$\\
    & & \diversity & $94.15_{0.43}$& $95.57_{0.36}$& $89.78_{0.40}$& $\mathbf{76.97_{0.47}}$& $\mathbf{90.68_{0.23}}$& $69.85_{1.60}$& $41.75_{2.16}$& $79.91_{1.14}$& $\mathbf{76.73_{0.61}}$\\
    & & \topkdiv & $94.43_{0.18}$& $\mathbf{96.59_{0.12}}$& $89.76_{0.16}$& $76.15_{0.41}$& $89.53_{0.23}$& $69.62_{0.56}$& $\mathbf{79.11_{0.97}}$& $78.39_{0.46}$& $75.91_{0.53}$\\
    \cline{2-12}
    & \multirow{4}{*}{8} 
    & \rand & $94.59_{0.45}$& $96.42_{0.46}$& $89.62_{0.21}$& $77.17_{0.69}$& $90.23_{0.25}$& $69.04_{0.53}$& $30.36_{2.04}$& $\mathbf{81.81_{0.37}}$& $77.09_{0.53}$\\
    & & \topk & $94.30_{0.32}$& $\mathbf{96.60_{0.18}}$& $90.14_{0.24}$& $77.20_{0.42}$& $89.95_{0.27}$& $66.54_{0.37}$& $77.54_{1.00}$& $80.72_{0.38}$& $76.42_{0.52}$\\
    & & \diversity & $\mathbf{94.61_{0.29}}$& $95.95_{0.27}$& $90.21_{0.23}$& $\mathbf{78.47_{0.41}}$& $\mathbf{90.45_{0.22}}$& $\mathbf{70.02_{0.91}}$& $46.96_{2.23}$& $81.38_{0.89}$& $\mathbf{77.84_{0.51}}$\\
    & & \topkdiv & $94.56_{0.29}$& $96.43_{0.13}$& $\mathbf{90.34_{0.22}}$& $77.29_{0.31}$& $89.70_{0.25}$& $68.48_{0.38}$& $\mathbf{82.39_{1.35}}$& $80.90_{0.53}$& $77.10_{0.54}$\\
    \bottomrule
    \end{tabular}
    \end{adjustbox}
    \end{table*}

\begin{table*}[htbp]
\centering
\small
\caption{\textbf{CodeLlama-family results on GeoQuery dataset with different split.} We observe that on GeoQuery dataset, \topkdiv consistently works better than \topk, and there is also a large gap between \topk and more diversity-aware methods like \diversity and \rand, which aligns with the results in \Cref{tab:Llama-family}, \Cref{tab:Gemma-family}, and \Cref{tab:main-res} for Mistral-v0.3. The gap between different methods is wide and $\mathrm{std}$ is small, so we omit the $\mathrm{std}$.}
\label{tab:CodeLlama-family}
\begin{tabular}{ccccccc}
\toprule
\multirow{2}{*}{model} & \multirow{2}{*}{$K$} & \multirow{2}{*}{Method}  & \multicolumn{4}{c}{GeoQuery} \\
& & & Standard & Tmcd & Template & Length \\
\midrule
\parbox[t]{3mm}{\multirow{8}{*}{\rotatebox[origin=c]{90}{CodeLlama-7B-hf}}} & \multirow{4}{*}{4} 
& \rand & 12.21& 10.43& 9.75& 3.61\\
& & \topk & 57.86& 35.68& 36.90& 25.91\\
& & \diversity & 33.11& 21.95& 27.22& 13.16\\
& & \topkdiv & \textbf{67.86}& \textbf{40.00}& \textbf{50.34}& \textbf{33.64}\\
\cline{2-7}
& \multirow{4}{*}{8} 
& \rand & 21.11& 17.25& 17.93& 8.05\\
& & \topk & 58.21& 42.95& 48.06& 32.95\\
& & \diversity & 38.29& 24.75& 31.41& 16.48\\
& & \topkdiv & \textbf{66.79}& \textbf{46.36}& \textbf{55.13}& \textbf{39.09}\\
\midrule
\parbox[t]{3mm}{\multirow{8}{*}{\rotatebox[origin=c]{90}{CodeLlama-13B-hf}}} & \multirow{4}{*}{4} 
& \rand & 13.82& 11.66& 11.73& 4.11\\
& & \topk & 63.57& 37.73& 38.04& 29.77\\
& & \diversity & 37.43& 23.23&  26.51& 18.07\\
& & \topkdiv & \textbf{72.14}& \textbf{44.32}& \textbf{53.99}& \textbf{40.68}\\
\cline{2-7}
& \multirow{4}{*}{8} 
& \rand & 24.89& 18.64& 21.16& 9.52\\
& & \topk & 69.64& 44.09& 56.04& 41.14\\
& & \diversity & 42.71& 26.00& 30.59& 20.68\\
& & \topkdiv & \textbf{79.29}& \textbf{47.73}& \textbf{64.24}& \textbf{44.32}\\
\midrule
\parbox[t]{3mm}{\multirow{8}{*}{\rotatebox[origin=c]{90}{CodeLlama-34B-hf}}} & \multirow{4}{*}{4} 
& \rand & 15.75& 13.02& 14.42& 5.98\\
& & \topk & 63.57& 42.05& 43.51& 30.23\\
& & \diversity & 39.86& 24.75& 32.92& 19.18\\
& & \topkdiv & \textbf{72.50}& \textbf{48.18}& \textbf{56.72}& \textbf{44.55}\\
\cline{2-7}
& \multirow{4}{*}{8} 
& \rand & 25.46& 20.50& 24.76& 11.73\\
& & \topk & 73.93& 48.41& 56.04& 44.09\\
& & \diversity & 44.18& 27.50& 39.29& 24.32\\
& & \topkdiv & \textbf{80.71}& \textbf{50.00}& \textbf{64.92}& \textbf{48.86}\\
\bottomrule
\end{tabular}
\end{table*}


We evaluated different model sizes from the Gemma and Llama families, including Llama-3.2-1B, Gemma-2-2B, Llama-3.2-3B, Llama-3.1-8B, Gemma-2-9B, Gemma-2-27B, and Llama-3.1-70B.
For math tasks, we used the instruct version of the corresponding models. For other tasks, we used the base models.
For code tasks, we also tested domain-specific CodeLlama models, including CodeLlama-7B-hf, CodeLlama-13B-hf, and CodeLlama-34B-hf. The results on CodeLlama were consistent with those of other base models.

We report the complete experimental results of the Llama family in Table \ref{tab:Llama-family}, the Gemma family results in Table \ref{tab:Gemma-family}, and the CodeLlama results in Table \ref{tab:CodeLlama-family}.
The methods that performed well on the corresponding tasks in Table \ref{tab:main-res} also demonstrated good performance across different model sizes.

\begin{table*}[!t]
    \centering
    \small
    \caption{Results of GPT-4o-mini and Deepseek-v3 in Math Task.}
    \label{tab:math of 4o-mini and v3}
    \begin{tabular}{ccccccc}
    \toprule
    \multirow{2}{*}{\centering Model} & \multirow{2}{*}{\centering $K$} & \multirow{2}{*}{Dataset} & \multicolumn{4}{c}{Method} \\
    & & & Rand & TopK & Div & TopK-Div \\
    \midrule
    \multirow{2}{*}{\centering GPT-4o-mini} 
    & \multirow{2}{*}{\centering 4} 
    & GSM8K & \textbf{93.03} & 91.06 & 92.80 & 92.27 \\
    & & PRM800K & 68.40 & 66.60 & \textbf{71.20} & 69.20 \\
    \midrule
    \multirow{2}{*}{\centering Deepseek-v3} 
    & \multirow{2}{*}{\centering 4} 
    & GSM8K & \textbf{96.13} & 95.75 & 95.91 & 95.45 \\
    & & PRM800K & 85.00 & \textbf{87.00} & 85.00 & 86.80 \\
    \bottomrule
    \end{tabular}
    \end{table*}

Due to resource constraints, our experiments primarily focused on mainstream open-source models. We tested the Math task on the commercial-grade models gpt-4o-mini and deepseek-v3. As shown in \Cref{tab:math of 4o-mini and v3}, the Div method consistently outperformed TopK on both gsm8k and prm800k.

\subsection{Changing the number of shots}\label{sec:ablation-shots}
In this section, we investigate how the performance advantage of diversity-aware methods over TopK evolves with increasing shot count. Our results in \Cref{fig:differ-k} show that the improvement from diversity-aware selection (Div) remains substantial even with higher number of shots.

We believe that as the shot number increases, there is an increase in redundant information among the examples selected by the \topk method. In contrast, the \topkdiv method minimizes the occurrence of redundant information as much as possible, thereby enabling the model to more clearly identify the task theme.

\begin{figure}
    \centering
    \includegraphics[width=\linewidth]{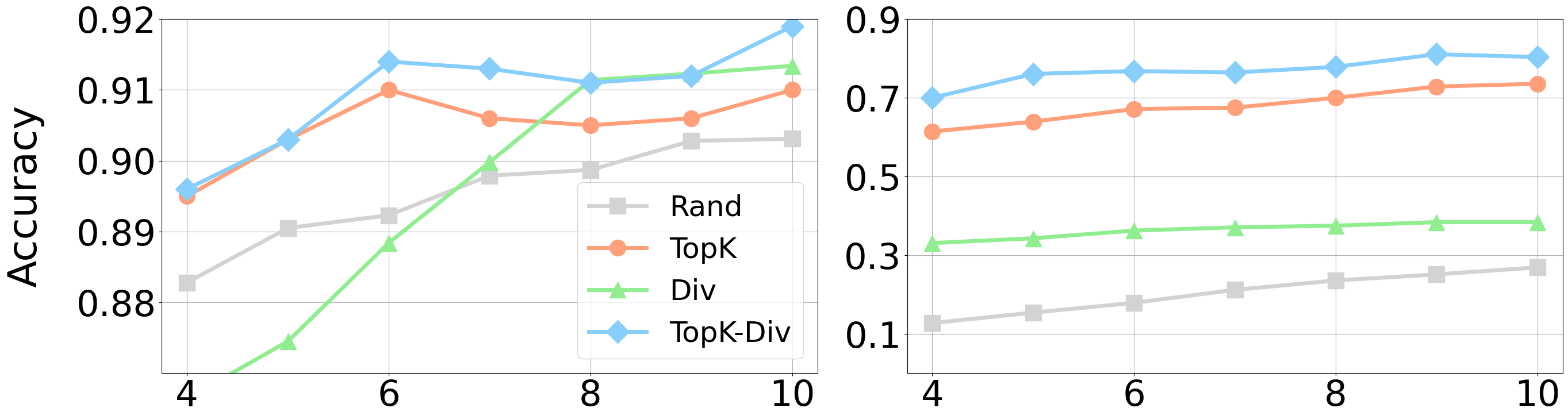}
    \caption{The performance of different demonstration selection methods with different number of shots $K$. \textbf{Left:} sentiment classification task with demonstrations come from Amazon and queries come from SST-2. \textbf{Right:} text to SQL task with demonstrations and query come from the training and test set of GeoQuery Standard Split.}
    \label{fig:differ-k}
\end{figure}

\Cref{fig:differ-k} presents the relative improvement on the GeoQuery standard split and SCIQ --- two tasks where diversity-aware methods showed clear benefits (\Cref{sec:main-findings}) --- across different sizes of the Llama-3.1/3.2 and Gemma-2 model families. The results indicate that, in general, the relative improvement from diversity-aware selection does not diminish significantly as model size increases. This underscores the continued importance of understanding diversity’s role in demonstration selection. 

\begin{figure}[!t]
\centering
    \includegraphics[width=0.7\linewidth]{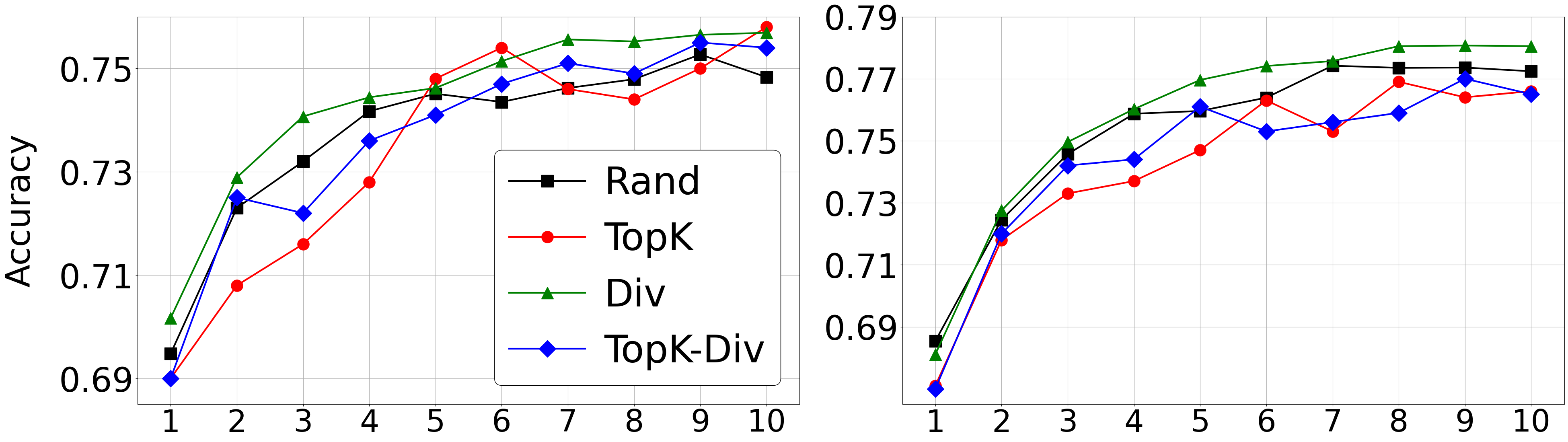}
    \caption{\textbf{(The accuracy of different methods with different number of shots $K$ on reading comprehension tasks.)} We choose report the results on Llama-3.1-8B, with Sentence-BERT embeddings (all-roberta-large-v1). \textbf{Left:} results where demonstration and query come from SCIQ. \textbf{Right:} results where demonstration and query come from SQuAD.}
    \label{fig:differ-k-sciq}
\end{figure}

We present the experimental results on reading comprehension task (SCIQ, SQuAD), where diversity-aware methods perform well, with different numbers of shots $K$ ranging from $1$ to $10$.  We test different methods on the Llama-3.1-8B model. \Cref{fig:differ-k-sciq} summarizes. To our surprise, when $k=1$, \topk performs significantly worse than \rand on both datasets, indicating that the accuracy of these datasets is not solely related to the coverage of example sets. Under most settings of $k$, \diversity shows significant advantages over \topk. Moreover, the correlation between example sets selected by \diversity and test samples is relatively low. This sufficiently demonstrates that even when example samples do not have coverage of test samples, they can still be high-quality examples, which is also consistent with the good performance of \rand.

\subsection{Results on subset size of Div} \label{sec:ablation-div-subset-size}


\begin{table*}[!t]
    \centering
    \small
    \label{tab:div_subset_size}
    \caption{Results of Different Subset Sizes for Div Method in Reading and Classification Tasks. For the Div method, we set $\text{subset size} = 100$ across all tasks. Additionally from Table 1 in our paper, we report new results for the Div method with $\text{subset size} = 8, 20, 50, 200, 500$. Embedding = all-roberta-large-v1.}
    \begin{adjustbox}{width=\textwidth}
    \begin{tabular}{ccccccccccc}
    \toprule
    \multirow{2}{*}{Model} & \multirow{2}{*}{Task} & \multirow{2}{*}{Dataset} & \multirow{2}{*}{$K$} & \multicolumn{7}{c}{Method} \\
    & & & & TopK & Div & Div-8& Div-20 &Div-50 & Div-200 & Div-500 \\
    \midrule
    \parbox[t]{4mm}{\multirow{8}{*}{\rotatebox[origin=c]{90}{Llama-3.1-8B}}} 
    & \multirow{4}{*}{Reading}
    & \multirow{2}{*}{SQuAD} 
    & 4 
    & $73.51$ & $75.66$ & $75.04$ & $75.42$ & $75.85$ & $\mathbf{75.93}$ & $73.74$\\
    & & & 8 
    & $75.52$ & $77.71$ & $77.00$ & $77.37$ & $\mathbf{77.89}$ & $77.19$ & $76.05$ \\
    \cmidrule{3-11}
    & & \multirow{2}{*}{SCIQ} 
    & 4 
    & $72.70$ & $74.47$ & $74.30$ & $73.77$ & $74.51$ & $\mathbf{74.67}$ & $74.06$\\
    & & & 8 
    & $74.72$ & $75.17$ & $75.55$ & $74.75$ & $75.11$ & $\mathbf{75.67}$ & $74.85$\\
    \cmidrule{2-11}
    & \multirow{4}{*}{Classification}
    & \multirow{2}{*}{SST-2} 
    & 4 
    & $\mathbf{94.13}$ & $91.50$ & $92.17$ & $83.99$ & $91.23$ & $92.28$ & $92.63$ \\
    & & & 8 
    & $\mathbf{93.64}$ & $92.95$ & $93.23$ & $90.98$ & $91.99$ & $92.95$ & $92.98$ \\
    \cmidrule{3-11}
    & & \multirow{2}{*}{Amazon} 
    & 4 
    & $96.24$ & $96.18$ & $\mathbf{96.61}$ & $95.94$ & $96.50$ & $96.50$ & $96.06$ \\
    & & & 8 
    & $96.12$ & $96.25$ & $\mathbf{96.95}$ & $96.54$ & $96.49$ & $96.38$ & $96.47$ \\
    \midrule
    \parbox[t]{4mm}{\multirow{8}{*}{\rotatebox[origin=c]{90}{Gemma-2-9B}}} 
    & \multirow{4}{*}{Reading}
    & \multirow{2}{*}{SQuAD} 
    & 4 & $74.82$ & $77.06$ & $76.91$ & $77.20$ & $\mathbf{77.40}$ & $77.15$ & $76.21$ \\
    & & & 8 & $77.59$ & $79.05$ & $79.07$ & $79.13$ & $\mathbf{79.44}$ & $78.86$ & $77.64$ \\
    \cmidrule{3-11}
    & & \multirow{2}{*}{SCIQ} 
    & 4 & $75.24$ & $75.96$ & $76.06$ & $76.23$ & $76.37$ & $\mathbf{76.47}$ & $76.20$ \\
    & & & 8 
    & $75.55$ & $76.65$ & $76.74$ & $76.81$ & $\mathbf{77.10}$ & $76.96$ & $76.21$ \\
    \cmidrule{2-11}
    & \multirow{4}{*}{Classification}
    & \multirow{2}{*}{SST-2} 
    & 4 & $\mathbf{94.47}$ & $93.45$ & $93.17$ & $91.42$ & $92.86$ & $93.68$ & $93.79$ \\
    & & & 8 & $\mathbf{94.20}$ & $93.41$ & $93.20$ & $92.31$ & $92.94$ & $93.76$ & $93.16$ \\
    \cmidrule{3-11}
    & & \multirow{2}{*}{Amazon} 
    & 4 & $96.34$ & $95.69$ & $\mathbf{96.39}$ & $96.23$ & $96.38$ & $96.15$ & $95.94$ \\
    & & & 8 
    & $\mathbf{96.55}$ & $95.94$ & $96.45$ & $96.36$ & $96.30$ & $96.10$ & $96.44$ \\
    \midrule
    \parbox[t]{4mm}{\multirow{8}{*}{\rotatebox[origin=c]{90}{Mistral-7B-v0.3}}} 
    & \multirow{4}{*}{Reading}
    & \multirow{2}{*}{SQuAD} 
    & 4 
    & $75.04$ & $75.96$ & $75.56$ & $\mathbf{76.29}$ & $76.22$ & $75.58$ & $74.13$ \\
    & & & 8 
    & $77.05$ & $77.67$ & $76.89$ & $\mathbf{77.89}$ & $77.68$ & $77.35$ & $75.38$ \\
    \cmidrule{3-11}
    & & \multirow{2}{*}{SCIQ} 
    & 4 & $73.73$ & $\mathbf{75.83}$ & $75.17$ & $75.05$ & $75.22$ & $75.34$ & $74.32$ \\
    & & & 8 
    & $75.44$ & $76.30$ & $76.10$ & $76.05$ & $75.74$ & $\mathbf{76.62}$ & $75.51$ \\
    \cmidrule{2-11}
    & \multirow{4}{*}{Classification}
    & \multirow{2}{*}{SST-2} 
    & 4 & $\mathbf{93.57}$ & $91.98$ & $91.28$ & $89.27$ & $91.24$ & $92.22$ & $92.50$ \\
    & & & 8 
    & $\mathbf{93.61}$ & $92.55$ & $92.77$ & $91.32$ & $91.57$ & $92.93$ & $92.88$ \\
    \cmidrule{3-11}
    & & \multirow{2}{*}{Amazon} 
    & 4 & $\mathbf{96.17}$ & $94.15$ & $94.59$ & $94.19$ & $94.97$ & $94.27$ & $94.83$ \\
    & & & 8 
    & $\mathbf{96.15}$ & $95.10$ & $96.00$ & $95.63$ & $95.61$ & $95.36$ & $95.90$ \\
    \bottomrule
    \end{tabular}
    \end{adjustbox}
\end{table*}

Here, we investigated the effect of different subset sizes on the Div method for both Classification and Reading tasks. The detailed results are presented in \Cref{tab:div_subset_size}. The optimal subset size is not consistent across different datasets, but this variation does not affect the original comparison results between the Div and TopK methods on these datasets. Therefore, to ensure a fair comparison across different datasets, we adopted a compromise in our main experiments by setting the subset size for the Div method to 100.

\subsection{More diversity-aware method}
\label{sec:ablation-k_means}
In the main text, \topkdiv and \diversity are both diversity-aware methods that combine the \topk method. We want to understand what happens when using a purely diversity-based method. Therefore, we implemented the \kmeans method:  dividing the training set into k clusters by k-means algorithm and then choose the nearest sample to the Centroid from each cluster (\kmeans), \kmeans can be viewed as a purely diversity-based met.

The results in \Cref{tab:k-means in Math task} show that the \kmeans method still has advantages compared to the \topk method. In fact, we believe \rand can also be considered a purely diversity-based method. This implies the advantage of diversity methods does not depend on the specific implementation.

\begin{table*}[!t]
    \centering
    \small
    \caption{Results of K-Means Method in Math Task. We add the K-Means baseline based on Table 1 in our paper. The implementation of K-Means consists of two steps: First, partition the input into $k$ clusters using the $k$-Means method. Second, select $k$ points as demonstrations by choosing the point closest to the cluster center within each cluster.}
    \label{tab:k-means in Math task}
    \begin{tabular}{cccccccc}
    \toprule
    \multirow{2}{*}{model} & \multirow{2}{*}{$K$} & \multirow{2}{*}{Dataset} & \multicolumn{5}{c}{Method}\\
    & & & Rand & Topk & Div & Topk-Div & K-means \\
    \midrule
    \parbox[t]{2mm}{\multirow{4}{*}{\rotatebox[origin=c]{90}{\tiny Llama-3.1-8B}}} 
    & \multirow{2}{*}{4} 
    & GSM8K & $82.24$ & $81.99$ & $82.14$ & $81.74$ & $\mathbf{83.89}$ \\
    & & GSM-Plus-M & $66.90$ & $65.30$ & $66.92$ & $66.12$ & $\mathbf{68.10}$ \\
    \cline{2-8}
    & \multirow{2}{*}{8} 
    & GSM8K & $82.81$ & $82.26$ & $\mathbf{82.98}$ & $82.63$ & $82.36$ \\
    & & GSM-Plus-M & $66.72$ & $65.99$ & $66.56$ & $66.48$ & $66.52$ \\
    \midrule
    \parbox[t]{2mm}{\multirow{4}{*}{\rotatebox[origin=c]{90}{\tiny Gemma-2-9B}}} 
    & \multirow{2}{*}{4} 
    & GSM8K & $84.29$ & $84.25$ & $84.44$ & $83.54$ & $\mathbf{85.24}$ \\
    & & GSM-Plus-M & $74.40$ & $74.50$ & $73.34$ & $74.47$ & $74.52$ \\
    \cline{2-8}
    & \multirow{2}{*}{8} 
    & GSM8K & $84.34$ & $83.57$ & $84.22$ & $83.85$ & $\mathbf{84.97}$ \\
    & & GSM-Plus-M & $74.48$ & $75.36$ & $74.69$ & $75.16$ & $\mathbf{76.29}$ \\
    \midrule
    \parbox[t]{2mm}{\multirow{4}{*}{\rotatebox[origin=c]{90}{\tiny Mistral-7B-v0.3}}} 
    & \multirow{2}{*}{4} 
    & GSM8K & $48.78$ & $49.28$ & $49.49$ & $\mathbf{49.99}$ & $43.90$ \\
    & & GSM-Plus-M & $37.20$ & $38.20$ & $37.50$ & $\mathbf{38.45}$ & $35.50$ \\
    \cline{2-8}
    & \multirow{2}{*}{8} 
    & GSM8K & $47.86$ & $48.43$ & $48.33$ & $48.60$ & $\mathbf{49.13}$ \\
    & & GSM-Plus-M & $36.32$ & $37.35$ & $36.12$ & $\mathbf{37.81}$ & $36.41$ \\
    \bottomrule
    \end{tabular}
\end{table*}

\subsection{Abalations on the size of training set} \label{sec:ablation-train-set-size}

To investigate whether the way diversity works is related to the size of the training set—for example, whether the example selection strategy needs to change when the available training set is limited, We conducted experiments on the SQuAD and SCIQ datasets by randomly sampling 50 examples from each training set to create SCIQ-50 and SQuAD-50, while keeping the original testing set unchanged. 

When the available training set size is reduced, TopK still underperforms compared to Div, maintaining an average performance gap of 1\% in 4-shot and 8-shot settings.

\subsection{Ablations on ``better'' embeddings}\label{sec:ablations-embs}

\begin{table}[!t]
\centering
\small
\caption{Embedding on answer using Gemma-2-9B with 4 shots. Comparing to \Cref{tab:main-res}, the relative ranking between the tested methods doesn't change.}
\label{tab:qwa}
\begin{tabular}{ccccc}
\toprule
 & \rand & \topk & \diversity & \topkdiv \\
\midrule
GSM8K & 82.21& 84.53& 84.14& \textbf{84.69}\\
PRM800K & 38.04& 45.60& 37.56& \textbf{46.40}\\
\midrule
GeoQuery(Standard) & 13.71& 79.64& 54.32& \textbf{84.29}\\
\bottomrule
\end{tabular}
\end{table}

\paragraph{``better'' embedding in a cheating way.} All methods we test, except randomly chosen (\rand), depend on an embedding model. It is always possible that the embedding model is not good enough. Indeed, using Sentence-BERT on questions/input (optimized for semantic similarity) might not be optimal for math tasks and text-to-SQL generation, and the ideal embedding might be highly dependent on the structure or reasoning steps of the answer. In this section, we test if diversity still helps when given a better embedding, computed in a ``cheating'' way: For math problems, we append the gold answer after the question and compute the embedding using Sentence-BERT; For text-to-SQL generation, we compute the occurrence of keywords in the answer~\citep{levy2023diverse}. \Cref{tab:qwa} summarizes the result using the ``cheating'' embeddings on Gemma-2-9B, and in general, diversity still helps for these tasks.


\begin{table*}[!t]
\centering
\small
\caption{\textbf{Results of different embeddings on Llama-3.1-8B.} We test different methods using different similarity scores computation (``all-roberta-large-v1'', ``BM25'', ``BertScore''). We test on Llama-3.1-8B model on math (using instruct model) and reading comprehension (using base model). The numbers for embedding ``all-roberta-large-v1'' are copied from \Cref{tab:main-res}. The numbers corresponding to \rand for BM25 and BertScore are also copied. We find that: (1) using another embedding might affect the \topk performance, as we can observe an increase of performance for \topk while changing to BM25 or BertScore. (2) Diversity still helps, since if we look at the best performance with the best embedding, in most of the cases the best performance is still achieved by diversity-aware methods.}
\label{tab:bm25-bert}
\begin{tabular}{ccccccc}
\toprule
\multirow{2}{*}{Embedding} & \multirow{2}{*}{$K$} & \multirow{2}{*}{Method}  & \multicolumn{2}{c}{Math} & \multicolumn{2}{c}{Reading}\\
& & & GSM8K & PRM800K & SQuAD & SCIQ \\
\midrule
\parbox[t]{4mm}{\multirow{8}{*}{\rotatebox[origin=c]{90}{all-roberta-large-v1}}} & \multirow{4}{*}{4} 
& \rand & 82.40& 43.50&  75.87& 74.17\\
& & \topk & \textbf{82.64}& 41.40& 73.70& 72.80\\
& & \diversity & 82.43& \textbf{44.86}& \textbf{76.02}& \textbf{74.44}\\
& & \topkdiv &81.43& 44.80 & 74.40& 73.60\\
\cline{2-7}
& \multirow{4}{*}{8} & \rand & 82.77& 43.32& 77.35& 74.79\\
& & \topk & 82.11& 43.00& 76.90& 74.40\\
& & \diversity & \textbf{83.13}& \textbf{44.28}& \textbf{78.05}& \textbf{75.52}\\
& & \topkdiv & 81.73& 40.00& 75.90& 74.90\\ 
\midrule
\parbox[t]{4mm}{\multirow{8}{*}{\rotatebox[origin=c]{90}{BM25}}} & \multirow{4}{*}{4} 
& \rand & 82.40& 43.50&  75.87& 74.17\\
& & \topk & 81.88 & 42.00 & 73.80 & \textbf{74.40} \\
& & \diversity & \textbf{82.47} & 44.12 & \textbf{76.65} & 72.74 \\
& & \topkdiv & 81.20 & \textbf{45.00} & 75.50 & 74.30 \\
\cline{2-7}
& \multirow{4}{*}{8} 
& \rand & 82.77& 43.32& 77.35& 74.79\\
& & \topk & 82.94 & 44.80 & 76.60 & \textbf{75.20} \\
& & \diversity & \textbf{83.44} & 43.92 & \textbf{78.97} & 74.12 \\
& & \topkdiv & 83.02 & \textbf{45.60} & 77.50 & 74.50 \\
\midrule
\parbox[t]{4mm}{\multirow{8}{*}{\rotatebox[origin=c]{90}{BertScore}}} & \multirow{4}{*}{4} 
& \rand & 82.40& 43.50&  \textbf{75.87}& 74.17\\
& & \topk & 81.58& \textbf{45.60}& 75.00& \textbf{74.30}\\
& & \diversity & \textbf{82.81}& 44.06& 74.16& 73.06\\
& & \topkdiv & 81.05& 44.40& 74.90& 73.20\\
\cline{2-7}
& \multirow{4}{*}{8} 
& \rand & 82.77& 43.32& \textbf{77.35}& 74.79\\
& & \topk & 82.34& 42.60& 76.40& \textbf{75.50}\\
& & \diversity & \textbf{83.09}& 43.68& 76.00& 74.95\\
& & \topkdiv & 81.58& \textbf{44.00}& 75.70& 74.70\\
\bottomrule
\end{tabular}
\end{table*}

\paragraph{Computing local structure for GeoQuery.} For the code-standard task, we tokenized the sample answers at the word level and obtained 52 distinct tokens, with each dimension representing a token.
For a given sample, in its 52-dimensional vector, if the corresponding token appears in its answer, the value at that position is 1, otherwise 0.
We use this embedding as the code embedding on answers.


\paragraph{BM25 and BertScore for math and reading comprehension.} We conduct ablation studies on the model to compute the similarity score, changing from cosine similarity from embeddings computed by ``all-roberta-large-v1'' to BM25 and BertScore, and test different methods on math and reading comprehension tasks. \Cref{tab:bm25-bert} summarizes our results.
We find that (1) using another embedding might affect the \topk performance, as we can observe an increase in performance for \topk while changing to BM25 or BertScore. (2) Diversity still helps since if we look at the best performance with the best embedding, in most cases, the best performance is still achieved by diversity-aware methods.




\subsection{Decoding method} \label{sec:ablation-decode-method}
\begin{table}[!h]
\centering
\small
\caption{Decode performance using Llama-3.1-8B on reading comprehension tasks. The number of shot is fixed as 4.}
\label{tab:decoding}

\begin{tabular}{cccccc}
\toprule
Decode & Test. & \rand & \topk & \diversity & \topkdiv \\
\midrule
\multirow{3}{*}{Greedy} & Squad & 75.87& 73.70& \textbf{76.02}& 74.40  \\
 & \hlcella Sciq & \hlcella 74.17& \hlcella 72.80& \hlcella \textbf{74.44}& \hlcella 73.60\\
\midrule
\multirow{2}{*}{Sampling} & Squad & 70.93& 72.40& 70.95& \textbf{72.80}\\
 & \hlcella Sciq &\hlcella 66.86&\hlcella 66.70&\hlcella 67.08&\hlcella \textbf{67.70}\\
\bottomrule
\end{tabular}
\end{table}

In this part we show some preliminary results on changing the decoding strategy for reading comprehension tasks (SQuAD and CommonsenseQA), since for code and math, greedy decoding is known to perform well. By changing greedy decoding to sampling decoding ($\text{topP} = 0.95, \text{Temperature}=0.7$), we find that the performance of all tasks drops a lot (\Cref{tab:decoding}), which justifies our decoding strategy selection.






\section{Omitted proofs of results in Section \ref{sec:theory}}
\label{sec:proof}

\subsection{Proof of Theorem \ref{thm:example_1}: justification example \rom{1}}
\label{sec:proof_1}

Fix a query $e_q$ drawn from $\gQ_\gE$.
By symmetry, we can assume the non-zero entry set of $e_q$ is $[2l]$.
For simplicity, we let $\theta = \theta_\gT$.

\paragraph{Demonstration example set for \topkdiv}
We first analyze the demonstration example set for \topkdiv, denoted by $S = \left\{s^{(1)}, s^{(2)}\right\}\subseteq D$.
Let $T^{(t)}$ denote the non-zero entry set of $s^{(t)}$.
By the construction of $\gD$, we first note that $|T^{(1)}\cap [l]| = \frac{l}{2}$.
By the rule of \topkdiv, we also note that $|T^{(2)}\cap [l]| = \frac{l}{2}$ and $T^{(1)}\cap T^{(2)} = \emptyset$.
Such $s^{(2)}$ must exist since all elements in the ground set of $\gD_\gE$ are contained in $D$, and is selected since it minimizes 
\[\alpha\cdot\texttt{Similarity}(e,e_q) + (1-\alpha)\texttt{Diversity}(e,S)\]
over all $e\in D-\left\{s^{(1)}\right\}$.

\paragraph{Demonstration example set for \topk}
Next, we compute the expected prediction loss $L$ for \topk.
Again, let its demonstration example set be $S = \left\{s^{(1)}, s^{(2)}\right\}\subseteq D$.
Let $T^{(t)}$ denote the non-zero entry set of $s^{(t)}$.
By the construction of $\gD$, we note that $|T^{(1)}\cap [l]| = |T^{(2)}\cap [l]| = \frac{l}{2}$.
However, different from the case of \topkdiv, $|T^{(1)}\cap T^{(2)}|$ can vary from $0$ to $l-1$.
To handle this, we define $a = |T^{(1)}\cap T^{(2)} \cap [l]|$ and $b = |T^{(1)}\cap T^{(2)} \cap ([d] \setminus [l])|$, and define $L_{a,b}$ to be the expected prediction loss conditioned on pair $(a,b)$.
Note that $0\leq a,b\leq l/2$ and $a+b\leq l - 1$.

\paragraph{Comparing $L$ and $L'$}
We remark that $L$ is a linear combination $\sum_{a,b} p_{a,b} L_{a,b}$ with $\sum_{a,b} p_{a,b} = 1$, where $p_{a,b}$ is the conditional probability with respect to intersection numbers $(a,b)$.
Also, $L' = L_{0,0}$.
By symmetry, we have the following observation:
\[
\Pr[a \leq l/4\leq b] \geq 0.25,
\]
where $l/4$ is the expectation of $a$ and $b$.
Thus, we have
\[
L \geq \sum_{a \leq l/4\leq b} p_{a,b} L_{a,b} \geq \sum_{a,b\in l/4\pm \sqrt{l}} p_{a,b} \cdot \min_{a \leq l/4\leq b} L_{a,b} \geq 0.25 \min_{a \leq l/4\leq b} L_{a,b}.
\]
Thus, to prove $L > L'$, it suffices to prove the following lemma.

\begin{lemma}[\bf{Comparing $L_{a,b}$ and $ L_{0,0}$}]
\label{lm:comparison_1}
For any $a \leq l/4\leq b$, we have $L_{a,b} > 4L_{0,0}$.
\end{lemma}

\begin{proof}
By symmetry, we assume $T^{(1)} = [\frac{l}{2}] \cup ([\frac{5}{2}l]- [2l])$, 
$T^{(2)} = ([l] - [a] - [\frac{l}{2}-a]) \cup ([3l-b] - [\frac{5}{2}l-b])$, 
$|T^{(1)} \cap T^{(2)} \cap [L]| = |T^{(1)} \cap T^{(2)} \cap [2L]| = a$, 
$|T^{(1)} \cap T^{(2)} \cap ([4L]-[2L])| = b$. 
The expected prediction loss for this setting equals $L_{a,b}$ since $\theta_i$s are i.i.d. random variables.
Let $\widehat{\theta}$ denote the min-norm solution defined as in Assumption \ref{ass:icl-lr}. 
Then we have

\begin{align} \label{eq:ex1_1}
\langle \widehat{\theta} - \theta, e_{T^{(1)}} \rangle = \sum\limits_{i=1}^{\frac{l}{2}}{\widehat{\theta}_i} + \sum\limits_{i=2l+1}^{\frac{5}{2}l} \widehat{\theta}_i -\sum\limits_{i=1}^{\frac{l}{2}}{\theta}_i - \sum\limits_{i=2l+1}^{\frac{5}{2}l} \theta_i = 0, 
\end{align}
and
\begin{align} \label{eq:ex1_2}
\langle \widehat{\theta} - \theta, e_{T^{(2)}} \rangle =\sum\limits_{i=\frac{l}{2}-a+1}^{l-a}{\widehat{\theta}_i} + \sum\limits_{i=\frac{5}{2}l-b+1}^{3l-b} \widehat{\theta}_i - 
\sum\limits_{i=\frac{l}{2}-a+1}^{l-a}\theta_i  - \sum\limits_{i=\frac{5}{2}l-b+1}^{3l-b} \theta_i = 0.
\end{align}

To get the min-norm solution, we need to minimize the following Lagrangian multiplier
\[
\mathcal{L}(\widehat{\theta},\lambda_1, \lambda_2) = \sum\limits_{i=1}^{l-a}\widehat{\theta}_i^2 - 2\lambda_1 \langle \widehat{\theta} - \theta, e_{T^{(1)}} \rangle - 2\lambda_2 \langle \widehat{\theta} - \theta, e_{T^{(2)}} \rangle.
\]
To ensure the partial derivatives with respect to $\widehat{\theta}$ equal to 0, we obtain that 
\begin{align}
\label{eq:ex1_3}
\begin{aligned}
& \quad \widehat{\theta}_1 = \ldots = \widehat{\theta}_{\frac{l}{2}-a} = \widehat{\theta}_{2l+1} = \ldots = \widehat{\theta}_{\frac{5}{2}l - b} = \lambda_1, \\
& \quad \widehat{\theta}_{\frac{l}{2}+1} = \ldots = \widehat{\theta}_{l-a}
= \widehat{\theta}_{\frac{5}{2}l+1}= \ldots = \widehat{\theta}_{3l-b} = \lambda_2, \\
& \quad 
\widehat{\theta}_{\frac{l}{2}-a+1} = \ldots = \widehat{\theta}_{\frac{l}{2}} 
= \widehat{\theta}_{\frac{5}{2}l-b+1} = \ldots = \widehat{\theta}_{\frac{5}{2}l} = {\lambda_1+\lambda_2}, \\
& \quad
\widehat{\theta}_{l-a +1} = \ldots = \widehat{\theta}_{2l} = \widehat{\theta}_{3l - b +1} =\ldots = \widehat{\theta}_{4 l} = 0.
\end{aligned}
\end{align}

Adding Equations (\ref{eq:ex1_1})-(\ref{eq:ex1_3}), we have
\begin{align}
\label{eq:ex1_4}
(l+a+b)(\lambda_1+\lambda_2) = \sum\limits_{i=1}^{\frac{l}{2}}{\theta}_i 
+\sum\limits_{i=2l+1}^{\frac{5}{2}l} \theta_i 
+\sum\limits_{i=\frac{l}{2}-a+1}^{l-a}\theta_i 
+\sum\limits_{i=\frac{5}{2}l-b}^{3l-b} \theta_i.
\end{align}

Thus, we conclude that
\begin{equation*}
\begin{aligned}
& \quad \left[\sum\limits_{i=1}^{l} \widehat{\theta}_i - \sum\limits_{i=1}^{l} \theta_i\right]^2  \\
=& \quad  \left[(\frac{l}{2}-a)(\lambda_1+\lambda_2) + a\lambda_1 + a\lambda_2 - \sum\limits_{i=1}^{l}\theta_i\right]^2 \\
= & \quad \left[\frac{l}{2}(\lambda_1+\lambda_2)- \sum\limits_{i=1}^{l} \theta_i\right]^2 \\
= & \quad \left[\frac{\frac{l}{2}}{l+a+b}
\left(\sum\limits_{i=1}^{\frac{l}{2}}{\theta}_i 
+\sum\limits_{i=2l+1}^{\frac{5}{2}l} \theta_i 
+\sum\limits_{i=\frac{l}{2}-a+1}^{l-a}\theta_i 
+\sum\limits_{i=\frac{5}{2}l-b}^{3l-b} \theta_i\right)
-\sum\limits_{i=1}^{l} \theta_i\right]^2 \\
= & \quad \left[-\frac{\frac{l}{2}+a+b}{l+a+b}
\sum\limits_{i=1}^{\frac{l}{2}-a}\theta_i -\frac{\frac{l}{2}+a+b}{l+a+b} \sum\limits_{i=\frac{l}{2}+1}^{l-a}\theta_i -\frac{a+b}{l+a+b}\sum\limits_{i=\frac{l}{2}-a+1}^{\frac{l}{2}}\theta_i\right.\\
&\quad\quad + \left.\frac{\frac{l}{2}}{l+a+b}\sum\limits_{i=2l+1}^{\frac{5}{2}l}\theta_i  +\frac{\frac{l}{2}}{l+a+b}\sum\limits_{i=\frac{5}{2}l-b+1}^{3l-b}\theta_i \right]^2,
\end{aligned}
\end{equation*}
where the first equation follows from Equation (\ref{eq:ex1_3}) and the third equation follows from Equation (\ref{eq:ex1_4}).
Since each $\theta_i$ is i.i.d. drawn from the uniform distribution over $[0,1]$, we have
\begin{align*}
L_{a,b} =& \E\left[\langle \widehat{\theta} - \theta, e_q \rangle^2\right]\\
=& \E\left[\left[\sum\limits_{i=1}^{l} \widehat{\theta}_i - \sum\limits_{i=1}^{l} \theta_i\right]^2\right]\\
=& \frac{(\frac{l}{2}+a+b)^2(\frac{l}{2}-a)+ \frac{a(a+b)^2}2 + \frac{l^3}{8}+\frac{3(bl-a^2-ab)^2}{2} }
{6 (l+a+b)^2}.
\end{align*}
Thus, $L_{0,0} = \frac{l}{24}$.
When $a \leq l/4\leq b$, we have
\begin{align*}
L_{a,b} > & \quad \frac{3(bl-a^2-ab)^2/2}{6(l+a+b)^2} & \\
\geq & \quad\frac{(l^2/4 - 2 (l/4)^2)^2}{4 (2l)^2} & (a \leq l/4\leq b) \\
= & \quad \frac{(l^2/8)^2}{16 l^2} & \\
= & \quad \frac{l^2}{1024} & \\
\geq & \quad 4L_{0,0}. & (l \geq 200)
\end{align*}
This completes the proof.
\end{proof}

\subsection{Proof of Theorem \ref{thm:example_2}: justification example \rom{2}}
\label{sec:proof_2}

By symmetric, we fix $e_q = e_{[l]}$.
Like the proof of Theorem \ref{thm:example_1}, we first study the demonstration example sets, denoted by $S = \left\{s^{(1)}, s^{(2)}\right\}\subseteq D$, derived from \topk and \topkdiv.
We observe that for both algorithms, $|T^{(1)}\cap [l]| = |T^{(2)}\cap [l]| = l - 1$.
Note that this property for \topkdiv follows from the choice of $\alpha \geq 1-\frac{1}{l}$, which ensures that $|T^{(2)}\cap [l]| \leq l - 2$ can not achieve the minimum for
\[
\alpha\cdot\texttt{Similarity}(e,e_q) + (1-\alpha)\texttt{Diversity}(e,S)
\]
Thus, by symmetry, we can fix $T^{(1)} = [l - 1]\cup \left\{2l +1]\right\}$ and there are only three choices for $T^{(2)}$:
\begin{itemize}
    \item Case 1: $T^{(2)} = [l]\cup \left\{2l + 2\right\} - \left\{1\right\}$;
    \item Case 2: $T^{(2)} = [l]\cup \left\{2l + 1\right\} - \left\{1\right\}$.
    \item Case 3: $T^{(2)} = [l - 1]\cup \left\{2l + 2\right\}$.
\end{itemize}

We define the expected prediction loss of these three cases to be $L_1, L_2, L_3$, respectively.
By the definition of \topkdiv, we know that $L' = L_1$.
Moreover, the expected prediction loss $L$ of \topk must be a linear combination of $L_1, L_2, L_3$.
Thus, it suffices to prove that $L_2 > L_1$ and $L_3 > L_1$.
Below, we compute $L_1, L_2, L_3$ separately.

\paragraph{Computing $L_1$.} 
The computation idea is similar to that of Lemma \ref{lm:comparison_1}.
Suppose $\widehat{\theta}$ is the min-norm solution and we have
\[
\sum\limits_{i=1}^{l-1} \widehat{\theta}_i + \widehat{\theta}_{2l+1} - \sum\limits_{i=1}^{l-1} \theta_i - \theta_{2l+1} = 0 \text{ and } \sum\limits_{i=2}^{l} \widehat{\theta}_i +\widehat{\theta}_{2l+2} - \sum\limits_{i=2}^{l} \theta_i - \theta_{2l+2} = 0. 
\]
Again, consider the Lagrangian multiplier $\mathcal{L}(\widehat{\theta},\lambda_1,\lambda_2) =  \sum\limits_{i=1}^{l}
 (\widehat{\theta}_i)^2 - 2\lambda_1 \langle \widehat{\theta} - \theta, e_{T^{(1)}} \rangle - 2\lambda_2 \langle \widehat{\theta} - \theta, e_{T^{(2)}} \rangle$. 
To ensure the partial derivative w.r.t. $\widehat{\theta}$ equal to 0, we have 
\[\widehat{\theta}_2 = \widehat{\theta}_3 = \ldots = \widehat{\theta}_{l-1} = \lambda_1+\lambda_2, \text{ and } \widehat{\theta}_1 = \widehat{\theta}_{2l+1} = \lambda_1, \widehat{\theta}_l = \widehat{\theta}_{2l+2} = \lambda_2.
\]
Combining the above equations, we have 
\[
(2l-2) (\lambda_1+\lambda_2) = 2\sum\limits_{i=2}^{l-1} \theta_i + \theta_1 + \theta_l + \theta_{2l+1} + \theta_{2l+2}.
\]
Thus, 
\[
(\sum\limits_{i=1}^l \theta_i - \sum\limits_{i=1}^l \widehat{\theta}_i)^2  = [(l-1)(\lambda_1+\lambda_2) - \sum\limits_{i=1}^l \theta_i]^2 = (\frac{\theta_{2l+1}+\theta_{2l+2} - \theta_{1} - \theta_{l}}{2})^2.
\]
Consequently, we have
\[
L_1 = \E[\langle \widehat{\theta} - \theta, e_q \rangle^2] = \E[(\frac{\theta_{2l+1}+\theta_{2l+2} - \theta_{1} - \theta_{l}}{2})^2] = \frac{1}{12}. 
\]

\paragraph{Computing $L_2$.}
Similarly, we have
\[
\sum\limits_{i=1}^{l-1} \widehat{\theta}_i + \widehat{\theta}_{2l+1} - \sum\limits_{i=1}^{l-1} \theta_i - \theta_{2l+1} = 0, \text{ and } \sum\limits_{i=2}^{l} \widehat{\theta}_i +\widehat{\theta}_{2l+1} - \sum\limits_{i=2}^{l} \theta_i - \theta_{2l+1} = 0. 
\]
Thus, using the Lagrangian multiplier, we obtain that 
\[
\widehat{\theta}_2 = \widehat{\theta}_3 = \ldots = \widehat{\theta}_{l-1} = \widehat{\theta}_{2l+1} = \lambda_1 + \lambda_2, \text{ and } \widehat{\theta}_{1} = \lambda_1, \widehat{\theta}_{l} = \lambda_2.
\]
Combining the above equations, we have
\[
(2l-1) (\lambda_1+\lambda_2) = 2\sum\limits_{i=2}^{l-1} \theta_i + 2\theta_{2l+1} + \theta_1+\theta_{l}.
\]
Thus,
\begin{equation*}
\begin{aligned}
L_2 = & \quad \E[(\sum\limits_{i=1}^l \theta_i - \sum\limits_{i=1}^l \widehat{\theta}_i)^2]  = 
\E[\sum\limits_{i=1}^l \theta_i - [(l-1)(\lambda_1+\lambda_2) ]^2] \\
= & \quad  \E[\frac{1}{2l-1}\sum\limits_{i=2}^{l-1}\theta_i + \frac{l}{2l-1}(\theta_1 +\theta_l) - \frac{2l-2}{2l-1} \theta_{2l+1}]^2] =\frac{9l^2-7l+2}{12(12l-1)^2} > L_1.
\end{aligned}
\end{equation*}

\paragraph{Computing $L_3$.}
Similarly, we have
\[
\sum\limits_{i=1}^{l-1} \widehat{\theta}_i + \widehat{\theta}_{2l+1} - \sum\limits_{i=1}^{l-1} \theta_i - \theta_{2l+1} = 0, \text{ and } \sum\limits_{i=1}^{l-1} \widehat{\theta}_i +\widehat{\theta}_{2l+2} - \sum\limits_{i=1}^{l-1} \theta_i - \theta_{2l+2} = 0. 
\]
Using the Lagrangian multiplier, we obtain that
\[
\widehat{\theta}_1 = \widehat{\theta}_2 = \widehat{\theta}_3 = \ldots = \widehat{\theta}_{l-1}=\lambda_1+\lambda_2, \text{ and } \widehat{\theta}_{2l+1} = \lambda_1,  
 \widehat{\theta}_{2l+2} = \lambda_2.
\]
Combining the above equations, we have 
\[
(2l-1) (\lambda_1+\lambda_2) = 2\sum\limits_{i=1}^{l-1} \theta_i + \theta_{2l+1} + \theta_{2l+2}.
\]
Thus,
\begin{align*}
L_3 = & \quad \E[(\sum\limits_{i=1}^l \theta_i - \sum\limits_{i=1}^l \widehat{\theta}_i)^2]  = 
\E[\sum\limits_{i=1}^l \theta_i - [(l-1)(\lambda_1+\lambda_2) ]^2]\\
= & \quad \E[\frac{1}{2l-1}\sum\limits_{i=1}^{l-1}\theta_i + \theta_l - \frac{l-1}{2l-1}(\theta_{2l+1}+\theta_{2l+2})]^2] = \frac{9l^2-7l+2}{12(2l-1)^2} > L_1.
\end{align*}

Overall, we complete the proof of Theorem \ref{thm:example_2}.

\section{Simulation} \label{sec:simulation}

We validate the superiority of \topkdiv compared to \topk in more general settings than that in Theorems \ref{thm:example_1} and \ref{thm:example_2}.

\subsection{Experiment settings}

We consider the ID setting with $\gD_\gE = \gQ_\gE$.

\paragraph{Metric for coverage.} 
Given a sample $(e,y_E)$, let $T^{(e)}$ denote the non-zero entry set of $e$. 
Given a demonstration example set $S\subseteq D$ and a query $e_q$, we define the coverage ratio of $S$ with respect to $e_q$ to be: 
\[
r_S(e_q) := \frac{|(\bigcup_{e\in S} T^{(e)})\cap T^{(e_q)}|}{|T^{(e_q)}|},
\]
i.e., the ratio of non-zero entries of $e_q$ covered by samples in $S$.
By definition, $r_S(e_q) \in [0,1]$ and a larger $r_S(e_q)$ represents higher coverage.
Specifically, when $r_S(e_q) = 1$, we say $e_q$ is fully covered by $S$.
Moreover, given a method $\gA$ that generates a demonstration example set $A(e_q) \subseteq D$ for each query $e_q$, we define
\begin{align}\label{eq:coverage}
r(\gA) := \E_{e_q\sim \gQ_\gE}[r_{\gA(e_q)}(e_q)]
\end{align}
to be the expected value of its coverage ratio $r_{\gA(e_q)}(e_q)$.
If $r(\gA) = 1$, we say every query is fully covered by $\gA$.

We want to study the loss difference between \topkdiv and \topk under two scenarios: 1) when query $e_q$ is fully covered by both algorithms \topkdiv and \topk, i.e., $r(\topkdiv) = r(\topk) = 1$; and 2) when the coverage ratio of \topkdiv is smaller than that of \topk, i.e., $r(\topkdiv) < r(\topk)$.

\paragraph{Parameters.}
Let $d=200$. 
Let $l$ vary from $3,4,8$.
Let $K = 4$ or $8$.
Let $\gD_\gE = \gQ_\gE$ be the distribution that first samples a subset $T\subset [d]$ of size $l$ and then generate $e_T$.
We set the size of training set $D$ to be $|D| = d \times\text{train\_scale}$, where $\text{train\_scale} \in \{1,5,10\}$.

For each pair $(l, K)$, we generate a testing set $D_{\text{test}}$ of size 100.
We ensure that $D_{\text{test}}\cap D = \emptyset$.
We report the expected prediction loss and coverage ratio of \topk and \topkdiv for each pair $(l, K)$.
%


\begin{table*}[htbp]
    \centering
    \small
    \caption{(\textbf{Simulation of the min-norm solution)} 
``Coverage'' represents the coverage ratio of methods, defined as in \Cref{eq:coverage}. For each random seed, we selected one hundred test samples. We report the average results across 3 different random seeds for each metric. }
    \label{tab:simulation}
    \begin{adjustbox}{width=\textwidth}
    \begin{tabular}{c|c|c|ccc|ccc|ccc}
    \toprule
    \multirow{2}{*}{Method} & \multirow{2}{*}{Shot} & \multirow{2}{*}{Metric} & \multicolumn{3}{c|}{Train scale = 1} & \multicolumn{3}{c|}{Train scale = 5} & \multicolumn{3}{c}{Train scale = 10} \\
    \cmidrule(lr){4-6} \cmidrule(lr){7-9} \cmidrule(lr){10-12}
    & & & $l = 3$ & $l = 4$ & $l = 8$ & $l = 3$ & $l = 4$ & $l = 8$ & $l = 3$ & $l = 4$ & $l = 8$ \\
    \midrule
    \multirow{4}{*}{TopK} & \multirow{2}{*}{$K=4$} 
    & Loss & 0.21 & 0.31 & 12.70 & 0.15 & 0.30 & 9.55 & 0.19 & 0.30 & 7.51 \\
    & & Coverage & 1.00 & 1.00 & 0.55 & 1.00 & 1.00 & 0.61 & 1.00 & 1.00 & 0.66 \\
    \cmidrule(lr){2-12}
    & \multirow{2}{*}{$K=8$}
    & Loss & 0.47 & 0.57 & 3.09 & 0.43 & 0.84 & 1.33 & 0.45 & 0.83 & 1.19 \\
    & & Coverage & 1.00 & 1.00 & 0.75 & 1.00 & 1.00 & 0.80 & 1.00 & 1.00 & 0.81 \\
    \midrule
    \multirow{4}{*}{TopK-Div} & \multirow{2}{*}{$K=4$} 
    & Loss & 0.19 & 0.32 & 10.25 & 0.18 & 0.31 & 5.47 & 0.21 & 0.29 & 3.97 \\
    & & Coverage & 1.00 & 1.00 & 0.63 & 1.00 & 1.00 & 0.75 & 1.00 & 1.00 & 0.80 \\
    \cmidrule(lr){2-12}
    & \multirow{2}{*}{$K=8$}
    & Loss & 0.31 & 0.38 & 2.58 & 0.23 & 0.38 & 1.32 & 0.20 & 0.38 & 1.75 \\
    & & Coverage & 1.00 & 1.00 & 0.87 & 1.00 & 1.00 & 0.94 & 1.00 & 1.00 & 0.94 \\
    \bottomrule
    \end{tabular}
    \end{adjustbox}
\end{table*}

\subsection{Result and discussions}
The results, reported in Table \ref{tab:simulation}, reveal key insights into the performance differences between \topk\ and \topkdiv.  
We observe that when \( l = 8 \), the coverage ratio of \topk\ is lower than that of \topkdiv, while its loss is significantly higher. 
For example, when \( l = 8 \), \( K = 4 \), and \(\text{train\_scale} = 5\), the coverage ratio is \( r(\topk) = 0.61 \), compared to \( r(\topkdiv) = 0.75 \), while the loss for \topk\ is 9.55, notably larger than the 5.47 observed for \topkdiv. 
This demonstrates that incorporating diversity can reduce prediction loss by improving coverage, aligning with Theorem \ref{thm:example_1}.  

When \( l = 3 \) or \( 4 \), the coverage ratios of \topk\ and \topkdiv\ are both 1. 
We find that the loss of \topk\ is comparable to or even lower than that of \topkdiv\ when \( K = 4 \), but significantly higher when \( K = 8 \), across various training scales. 
For instance, when \( l = 3 \), \( K = 8 \), and \(\text{train\_scale} = 5\), the loss for \topk\ is 0.43, whereas for \topkdiv\ it is 0.23. 
This supports our findings in Theorem \ref{thm:example_2}, demonstrating that diversity can enhance in-context learning beyond just coverage.  
The inverse trend in loss between \( K = 4 \) and \( K = 8 \) suggests that increasing coverage is beneficial when the query is not fully covered but becomes redundant when the demonstration example set already provides sufficient coverage.

\end{document}